\pdfoutput=1
\documentclass[nohyperref]{article}

\usepackage{microtype}
\usepackage{graphicx}
\usepackage{booktabs} 

\usepackage{hyperref}

\usepackage[numbers,sort&compress]{natbib}
\usepackage[accepted]{icml2022}

\usepackage{amsmath}
\usepackage{amssymb}
\usepackage{mathtools}
\usepackage{amsthm}
\usepackage[capitalize,noabbrev]{cleveref}

\usepackage[textsize=tiny]{todonotes}

\usepackage{pifont}

\usepackage{xcolor} 
\usepackage{graphicx} 
\usepackage{booktabs} 
\usepackage{array} 
\usepackage{mathtools} 
\usepackage{amsthm, amssymb, amscd, amsfonts} 
\usepackage[ampersand]{easylist} 
\usepackage{multirow} 
\usepackage{booktabs} 
\usepackage{bm} 
\usepackage{bbm} 
\usepackage{dcolumn} 
\usepackage{caption}
\usepackage{subcaption} 
\usepackage{epstopdf} 
\usepackage{comment} 
\usepackage[export]{adjustbox}
\usepackage{hyperref}
\usepackage{wrapfig}
\usepackage{thm-restate} 
\usepackage{xspace} 
\usepackage[normalem]{ulem}

\usepackage{enumitem}
\setitemize{noitemsep,topsep=0pt,parsep=0pt,partopsep=0pt}
\setenumerate{noitemsep,topsep=0pt,parsep=0pt,partopsep=0pt}
\usepackage{hyperref}
\usepackage{algorithm} 
\usepackage{algorithmic} 
\usepackage{qtree}

\usepackage[edges]{forest}
\usepackage{tabularx} \newtheorem{theorem}{Theorem}
\newtheorem{lemma}{Lemma}

\newtheorem{definition}{Definition}

\theoremstyle{definition}

\theoremstyle{definition}

\theoremstyle{definition}

\theoremstyle{definition}

\theoremstyle{definition}

\newcommand{\equivtree}{\equiv_{\textnormal{tr}}}

\newcommand{\nodes}{\mathcal{N}}

\newcommand{\node}{N}

\newcommand{\domain}{\mathcal{D}}
\newcommand{\range}{\mathcal{R}}
\newcommand{\nodedomain}{\mathcal{D}(\node)}
\newcommand{\depth}{\textnormal{depth}}
\newcommand{\counts}{\tilde{c}}
\newcommand{\modpi}{\tilde{\pi}}
\newcommand{\anc}{\textnormal{anc}}
\newcommand{\desc}{\textnormal{desc}}
\newcommand{\maxdepthparam}{M}
\newcommand{\minsplit}{n_{\min}}

\newcommand{\new}{\textnormal{new}}
\newcommand{\init}{\textnormal{\,init}}

\newcommand{\piall}{\Sigma}

\newcommand{\pindp}{\Pi}

\newcommand{\lleft}{\textnormal{left}}
\newcommand{\rright}{\textnormal{right}}

\renewcommand{\algorithmicrequire}{\textbf{Input:}}
\renewcommand{\algorithmicensure}{\textbf{Output:}}

\DeclareMathOperator*{\argmin}{arg\,min\,}
\DeclareMathOperator*{\argmax}{arg\,max\,}
\DeclareMathOperator*{\Id}{id}

\newcommand{\Z}{\ensuremath{\mathbb{Z}}}

\newcommand{\nobs}{n}

\newcommand{\nvar}{d}
\newcommand{\ndim}{\nvar}

\newcommand{\x}{x}
\newcommand{\xvec}{\bm{\x}}

\newcommand{\y}{y}
\newcommand{\yvec}{\bm{\y}}

\newcommand{\z}{z}
\newcommand{\zvec}{\bm{\z}}
\newcommand{\zmat}{Z}
\newcommand{\zset}{\mathcal{\zmat}}

\newcommand{\E}{\mathbb{E}}

\newcommand{\one}{\mathbbm{1}}

\newcommand{\tree}{\mathcal{T}}

\newcommand{\Xdata}{X}

\newcommand{\wrapproof}[1]{
\iftrue
#1
\fi
}

\newcommand{\treepath}{\mathcal{P}}
\newcommand{\rcset}{\mathcal{R}}

\begin{document}

\twocolumn[
\icmltitle{Discrete Tree Flows via Tree-Structured Permutations}

\icmlsetsymbol{equal}{*}

\begin{icmlauthorlist}
\icmlauthor{Mai Elkady}{equal,yyy}
\icmlauthor{Jim Lim}{equal,comp}
\icmlauthor{David I. Inouye}{comp}

\end{icmlauthorlist}

\icmlaffiliation{yyy}{Department of Computer Science, Purdue University }
\icmlaffiliation{comp}{Department of Electrical and Computer Engineering, Purdue University}

\icmlcorrespondingauthor{Mai Elkady}{melkady1@purdue.edu}
\icmlcorrespondingauthor{Jim Lim}{lim316@purdue.edu}

\icmlkeywords{Machine Learning, Flow models, Discrete Generative models, ICML}

\vskip 0.3in
]

\printAffiliationsAndNotice{\icmlEqualContribution} 

\begin{abstract}
While normalizing flows for continuous data have been extensively researched, flows for discrete data have only recently been explored. These prior models, however, suffer from limitations that are distinct from those of continuous flows. Most notably, discrete flow-based models cannot be straightforwardly optimized with conventional deep learning methods because gradients of discrete functions are undefined or zero. Previous works approximate pseudo-gradients of the discrete functions but do not solve the problem on a fundamental level. In addition to that, backpropagation can be computationally burdensome compared to alternative discrete algorithms such as decision tree algorithms.  Our approach seeks to reduce computational burden and remove the need for pseudo-gradients by developing a discrete flow based on decision trees---building upon the success of efficient tree-based methods for classification and regression for discrete data. We first define a tree-structured permutation (TSP) that compactly encodes a permutation of discrete data where the inverse is easy to compute; thus, we can efficiently compute the density value and sample new data. We then propose a decision tree algorithm to build TSPs that learns the tree structure and permutations at each node via novel criteria. We empirically demonstrate the feasibility of our method on multiple datasets.
\end{abstract}

\section{Introduction}

 Discrete categorical data is abundant in numerous applications and domains, from DNA sequences and medical records to text data and many forms of tabular data. Analyzing discrete data by modelling and inferring its distribution is crucial in many of those applications and can encourage innovation in discrete data utilization such as generating new data that is of similar distribution and properties to our original data. However, discrete data can be inherently hard to model. 

Probabilistic graphical models are an important classical way to model discrete distributions such as the Ising model for binary data \citep{wainwright2008graphical}, multivariate Poisson distributions \citep{inouye2017review} for count data, mixture models \citep{ghojogh2020fitting}, admixture models \citep{inouye2014admixture}, and discrete variational autoencoders (VAE) \citep{oord2018neural}, \citep{razavi2019generating}.
However, some of these graphical models lack tractable exact likelihood computation or sampling procedures.

Autoregressive models for discrete data can provide exact likelihood computation (e.g., \cite{germain2015made}), however, their structures hinder sampling speed as variables are conditionally dependent on one another and must be traversed in a fixed order.

A relatively new approach in tackling this problem is normalizing flows \cite{rezende2016variational}, which leverage invertible models to combine different advantages of latent variable models by having a latent variable representation of the data, while also allowing for exact likelihood computation and fast sampling. These models have mostly focused on continuous random variables \cite{dinh2017density},\cite{NIPS2016_ddeebdee},\cite{papamakarios2018masked},
but have recently been introduced for discrete random variables as well. 

The two key components of discrete flows are a base distribution (similar to continuous flows) and a permutation of the discrete configuration values (a discrete version of invertible functions) \cite{tran2019discrete}.
\emph{Discrete flows} have a similar change of variable structure via invertible models as seen in this comparison between continuous change of variables and discrete change of variables:
\begin{align}
\begin{aligned}
    P_x(x) &= Q_z(f(x))\det{\left|\frac{df(x)}{dx} \right|} \quad \textnormal{(Continuous)} \\
    P_x(x) &= Q_z(f(x)) \quad \textnormal{(Discrete)} 
\end{aligned}
    \label{eqn:change-of-variables-both}
\end{align}
where $Q_z$ is some (possibly learnable) base distribution that is usually a simple distribution (e.g., a Gaussian in the continuous flows case or an independent categorical distribution for discrete flows), and $f$ is an invertible function.
Importantly, in the discrete case, the only invertible functions are permutations over the possible discrete configurations of $x$; thus, there is no Jacobian determinant term because permutations do not change volume.
However, given that the number of discrete configurations of $d$ features with $k$ possible discrete values is $k^d$, the number of all possible permutations of all configurations is $k^d!$.
Optimizing over all possible permutations is computationally intractable.
Thus, parameterizing and optimizing over a subset of possible permutations is critical for practical algorithms and generalizability.

Different methods were proposed to handle modeling discrete data using normalizing flows, among these are:

\paragraph{Using a straight-through gradient estimator (STE) to approximate the gradients:} \citet{tran2019discrete} introduces Autoregressive and Bipartite flows (AF and BF) to model \emph{categorical} discrete data and proposes to parameterize the permutations using neural networks, and to tackle discrete functions non-differentiability they resorted to using an STE \cite{bengio2013estimating} along with a Gumbel-softmax distribution for back-propagation. This work suggests that gradient bias from the discrete gradient approximations may be a key issue, \citet{berg2020idf} later show that the architecture of the coupling layers is significantly more important than the gradient bias issue.
\citet{hoogeboom2019integer} introduces integar discrete flows (IDF) and \citet{berg2020idf} improves upon it in IDF++. Both works also use STEs but focus on \emph{integer} discrete data.

\paragraph{Using a mixture of continuous and discrete training:} \citet{lindt2021discrete} introduced Discrete Denoising Flows (DDFs) to model \emph{categorical} discrete data. They proposed to separate learning into a continuous learning step and a discrete step: learn an NN probabilistic classifier on half of the features to predict the other half of the features (i.e., a bipartite coupling layer) and then sort each feature in the second half based on the predicted logits (i.e., the discrete operation). This approach sidesteps the possible gradient bias.

\paragraph{Solving the problem in the continuous space, then projecting the solution in the discrete space:}

\citet{ziegler2019latent} introduced 3 different flow architectures (which we'll refer to as latent flows (LF)) that can be used as the prior in a VAE model with a variable length latent code which is used to encode the discrete data in a continuous space. In a similar fashion, \citet{lippe2021categorical} introduces categorical normalizing flows (CNF) and proposes to use an encoder to project categorical discrete variables to the continuous space using variational inference and then utilize a continuous flow model to solve the problem. \citet{hoogeboom2021argmax} introduced ARGMAX flows (ARGMAXF) that use an argmax based transformation, and a probabilistic right inverse to lift the discrete categorical data into the continuous space. These approaches don't allow for a discrete latent space to be achieved, nor for exact likelihood calculation, but overcomes issues that are common for categorical data when using dequantization. A comparison of all these previous works is presented in \autoref{tab:methodss_comparisons}.

\begin{table}[t]
\caption{A comparison of different methods for solving the discrete flows problem in terms of whether the method handles categorical data, will produce a discrete latent space and is trained by Exact likelihood Optimization(ELO)}
\label{tab:methodss_comparisons}
\begin{center}
\begin{small}
\begin{sc}
\begin{tabular}{p{0.18\linewidth}p{0.2\linewidth}p{0.2\linewidth}p{0.2\linewidth}}
\toprule
 & \shortstack{Categorical\\ data}  & \shortstack{Discrete \\latent space} & \shortstack{Training by\\ ELO} \\
\midrule
 AF, BF & \hfil $\surd$  & \hfil $\surd$  & \hfil $\surd$  \\ 
 DDF & \hfil$\surd$  & \hfil $\surd$  & \hfil $\surd$  \\
 IDF & \hfil $\times$ & \hfil $\surd$  & \hfil $\surd$  \\
 IDF++ & \hfil $\times$ & \hfil $\surd$  & \hfil $\surd$  \\
 LF  & \hfil $\surd$  & \hfil $\times$ & \hfil $\times$ \\
 CNF & \hfil $\surd$  & \hfil $\times$ & \hfil $\times$ \\
 ARGMAXF & \hfil $\surd$  & \hfil $\times$ & \hfil $\times$ \\  
\bottomrule
\end{tabular}
\end{sc}
\end{small}
\end{center}
\vskip -0.33in
\end{table}
Most of these prior methods don't address the discrete nature of the problem on a fundamental level and may be computationally expensive (as they usually involve optimizing many neural network parameters which translates to an increase in training times) compared to alternative discrete-oriented algorithms such as those based on decision trees---which have seen wide success in classification and regression for discrete data (e.g., XGBoost \citep{chen2016xgboost}).
Thus, we seek to answer the following research question: \textbf{Can we design a more computationally efficient discrete flow algorithm using decision trees that handles discrete data on a fundamental level?}

To answer this, we propose a novel tree-structured permutation (TSP) model that can compactly represent permutations and a novel decision tree algorithm that optimizes over the space of these permutations.
Moreover, for more powerful permutations, we can iteratively build up a sequence of TSPs to form a deep permutation which we refer to as Discrete Tree Flows (DTFs).
We summarize our contributions as follows:
\begin{itemize}
    \item We define Tree-Structured Permutations (TSP) that compactly encode permutations at each node of the tree and we prove what constraints are required on these permutations such that the whole TSP is efficiently invertible.

    \item We also propose a novel decision tree algorithm for building TSPs that includes building the tree structure based on a splitting critera, and two passes to learn and apply permutations. We theoretically prove the viability of our algorithm as well.
    
    \item Finally, we demonstrate the feasibility of our method on simulated and real-world categorical datasets.

\end{itemize}

\section{Model: Discrete Tree Flows (DTF)}

We introduce a new model for discrete flows called Tree Structured Permutation (TSP) that utilizes trees for compactly parameterizing a set of permutations based on decision trees.
Our Discrete Tree Flow (DTF) model is merely a composition of multiple TSPs. 
First, we briefly define some notation that will be used throughout the paper.

\paragraph{Notation}
We will denote a discrete dataset as $\Xdata \in \mathcal{Z}^{n\times d}$ where $n$ is the number of samples, $d$ is the number of dimensions, and $\mathcal{Z}$ is a set of discrete values, and where the maximum number of possible discrete values per feature (i.e., the number of categories) is $k$.
Let $\pi$ denote an \emph{independent} permutation, i.e., $\pi(\xvec) = [\pi_0(x_0), \pi_1(x_1), \cdots, \pi_d(x_d)]$ where $\pi_j(x_j)$ is a 1D permutation, and let $\pindp$ denote the set of independent permutations.
Similarly, let $\sigma$ denote a general (possibly non-independent) permutation (e.g., $\sigma_{\tree}$), and let $\Sigma$ denote a set of general (possibly non-independent) permutations.
Given an independent permutation $\pi$ and a count matrix $c \in \Z^{d \times k}$, let $\pi[c]$ denote the operation of permuting the entries of the count matrix by $\pi$ (note that this is different than applying the permutation to a category value $a$ as in $\pi(a)$) and can be defined as:
$\pi[c] \triangleq [\pi_0[c(0)], \pi_1[c(1)], \cdots, \pi_d[c(d)]]\,,$
where
$\pi_j[c(j)] \triangleq [c(j, \pi_j^{-1}(0)), c(j, \pi_j^{-1}(1)), \cdots, c(j, \pi_j^{-1}(k))] \,$, and where the first entry and the second entry in $c(\cdot,\cdot)$ specify a feature and a category, respectively in the count matrix.
We let $\circ$ denote the composition operator, e.g., $\pi_1 \circ \pi_2(x) = \pi_1(\pi_2(x))$.
We will denote a binary decision tree by it's set of nodes $\tree \triangleq \{ \node_j \}_{j=0}^{|\nodes|-1}$ where each node (except a leaf node) has two child nodes and $\node_0$ is the root node.
Each node will have an associated split feature $s$, split value(s) $v$, and node permutation $\pi_{\node}$.

\subsection{Tree-Structured Permutations}
We define a \emph{tree-structured permutation} to be a binary decision tree where each node $\node$ is described by both a permutation $\pi$ and the usual split information (i.e., a split feature $s$ and split value $v$).
Informally, to evaluate a TSP, an input vector traverses the tree from the root to a leaf node based on the split information, and the node permutations are applied as soon as the data reaches the node.
Thus, a TSP node will first permute the input data and then forward the data to the left or right node depending on the split information.
To ensure the TSP is computationally tractable, we choose to restrict our allowable permutations to the natural and computationally tractable class of independent feature-wise permutations, denoted by $\pindp$, which allows each feature to be permuted independently of the other features (i.e., the permutation of one feature cannot depend on the permutations of other features).
This class significantly reduces the number of permutations compared to all possible permutations, i.e., $|\pindp| = (k!)^d  \ll  (k^d)! = |\piall|$, where $\piall$ is the set of all possible permutations of $k^d$ possible configurations.
We illustrate the forward traversal with an example in Figure~\ref{fig:demo_1}.
We formally define our tree-structured permutation as follows:
\begin{definition}[Tree-Structured Permutation]
A \emph{tree-structured permutation} is defined as $\sigma_{\tree}(\xvec) \triangleq f_{\node_0}(\xvec)$, where
$\node_0$ is the root node and 
\begin{align}
    f_\node(x) &\triangleq \begin{cases}
    \pi_\node(x), & \textnormal{if leaf node} \\
    f_{\node_{\lleft}}(\pi_\node(x)), & \textnormal{if }\, \pi_{\node}(x)_s \in v \\
    f_{\node_{\rright}}(\pi_\node(x)), & \textnormal{otherwise} 
    \end{cases}
\end{align}
where $s \in \{0,1,\dots,k-1\}$ denotes the split feature for the node and $v$ are the set of values that determine if the observation goes left in the tree.
\end{definition}

We also define the set of configurations that will reach a particular node as follows:
\begin{restatable}[Node Domain]{definition}{NodeDomain}
\label{def:node-domain}
The \emph{node domain} of the root node is all possible configurations, i.e., $\domain(\node_0) \triangleq \zset^\ndim$, and \emph{node domains} of children nodes are defined recursively as:
\begin{align*}
    \domain(\node_{\lleft}) &\triangleq \{ \pi_{\node}(\xvec): \xvec \in \domain(\node), \pi_{\node}(\xvec)_s \in v\} \\
    \domain(\node_{\rright}) &\triangleq \{ \pi_{\node}(\xvec): \xvec \in \domain(\node), \pi_{\node}(\xvec)_s \not\in v\} \,,
\end{align*}

where $s$ and $v$ are the split feature and value(s) of node $\node$.
Similarly, let
$\domain_j(\node) \triangleq \{x_j : \xvec \in \domain(\node) \}$ denote the domain of the $j$-th feature.
\end{restatable}

\begin{figure}[!ht]
\centering
\includegraphics[trim=1cm 0cm 1cm 1cm, clip, width=\linewidth ]{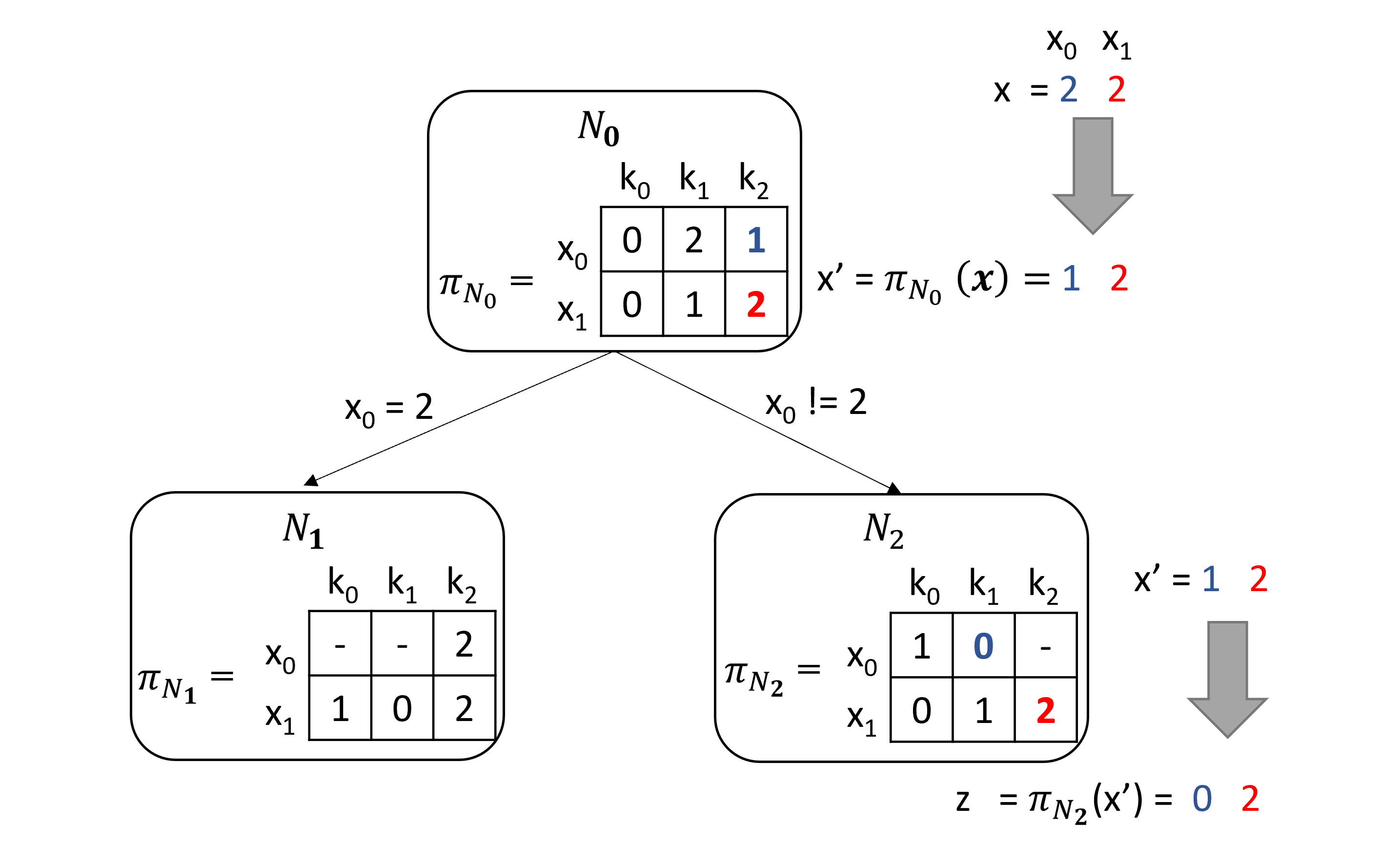}
\vspace{-1em}
\caption{An example of evaluating the forward pass of a TSP on datapoint x yielding datapoint z. At each node the input data will be permuted by the node's permutation $\pi_{\node}$.
We will represent the independent permutation $\pi$ by a $d \times k$ matrix (which is different from a regular permutation matrix).
The row indices correspond to each feature ($x_0$ and $x_1$), the column indices correspond to categories ($k_0, k_1, k_2$), and the matrix entries correspond to the new category value (i.e., the permuted value).
A value of {-} in the matrix indicates entries that are outside the node's domain $\domain(\node)$. 
As an example, consider applying the permutation to the input $\xvec$.
The first dimension ($x_0$) has a category value 2. In the permutation representation matrix the entry at row index 0 and column index 2 is 1, so the datapoint will be permuted to 1 at the 0th dimension after passing through $\node_0$. The data then passes to the left or right nodes depending on the split information which are presented on the arrows.}
\label{fig:demo_1}
\vspace{-1em}
\end{figure}

\subsection{Invertibility of TSPs}
\label{Invertibility}

To ensure our TSPs are invertible (and thus applicable to discrete flows), we prove that a simple and intuitive constraint on the node permutations is sufficient as defined next.

\begin{theorem}
\label{thm:invertibility-constraint}
TSP Invertibility Constraint:
A TSP is invertible if the range of the node permutations is equal to the node domain, i.e., $\forall \xvec \in \domain(\node), \pi_{\node}(\xvec) \in \domain(\node)$.
\end{theorem} 

The proof is constructive and relies on the following lemma that is proved by induction in Appendix~\ref{inv_proofs}.

\begin{restatable}{lemma}{RecoverabilityOfTreePath}
\label{lem:recoverability-of-tree-path}
If the invertibility constraint is satisfied, the TSP tree traversal path for any input can be recovered from the output.
\end{restatable}

Given this lemma, the sequence of permutations that was applied to an input point can be recovered and the inverse is then merely the inverse of this sequence of permutations (which are all invertible themselves).

While this gives us understanding about the invertibility of TSPs, the proof of \cref{lem:recoverability-of-tree-path} (in Appendix~\ref{inv_proofs}) naturally gives an efficient algorithm for computing the inverse.
We show that we can compute the inverse by traversing the tree in the forward direction (i.e., from root to leaves) without applying the node permutations but keeping track of the path.
Once we reach a leaf, we apply the permutations along the path in reverse order to compute the inverse.
While we provide a simple sufficient condition for invertibility, we prove a more complex necessary and sufficient condition for invertibility in Appendix~\ref{inv_proofs}.

\subsection{Expressivity of DTFs }
We prove in Appendix~\ref{universal_proof} that a composition of TSPs does not restrict expressivity, i.e., a sequential composition of TSPs can produce a universal permutation.
As a proof sketch, we demonstrate that a single TSP can swap configurations that differ in only one feature value while keeping all other configurations the same.
Then, we prove that there exists a snake-like path that connects all possible configurations such that adjacent configurations on the path differ in only one feature value.
Finally, we use the fact that all transpositions (i.e., swaps) between two adjacent configurations is a generating set of the whole permutation space.
Thus, by composing TSPs into a DTF we can express any possible permutation (see Appendix~\ref{universal_proof} for full proof).

\section{Algorithm: Learning Discrete Tree Flows}
Our goal is to minimize the Negative Log Likelihood (NLL) assuming an independent base distribution $Q_{\zvec}$, i.e., 

\begin{align}
    \argmin_{\sigma_{\tree}} \min_{Q_{\zvec}} -\frac{1}{n} \sum_{i=1}^\nobs \log Q_{\zvec}(\sigma_{\tree}(\xvec_i)) \,.
\end{align}
Note that solving the inner minimization problem over $Q_{\zvec}$ is known in closed-form based merely on the discrete value counts along each dimension independently.
While at first it may seem that we must learn the permutations and the splits simultaneously, we show that we can decouple this problem into two subproblems: 1) Estimate decision tree structure and 2) Estimate node permutations.

First, our algorithm estimates the decision tree structure (i.e., the split features and values) via a split criteria.
Second, given the decision tree structure, our algorithm globally optimizes the node permutations over all TSPs with equivalent tree structure (the equivalence of tree structure is defined later). Intuitively, this finds the node permutations that will make the different conditional distributions align as much as possible---thereby moving the joint distribution towards independence. 
We give the pseudocode of our algorithms in Appendix~\ref{code}. 

\subsection{Splitting Criteria for Structure Learning}
\label{learning_splits}
As in every decision tree algorithm, we need to define an appropriate splitting criteria.
We choose to implement a simple baseline of randomly splitting, i.e., choosing a split feature at random then a split value also at random. We refer to a DTF that uses this approach as $\textnormal{DTF}_{RND}$. 
We also propose another approach for splitting, 
that is a heuristic we call \emph{greedy local permutation} (GLP), that relies on theorizing about doing a permutation for a specific split and choosing the split that leads to the best decrease in NLL if that hypothetical permutation is to be applied. 

\paragraph{Greedy local permutation splitting criteria}
For this approach, we evaluate the possible splits based on the potential decrease in NLL that will be attained if we were to split the data accordingly and do a local greedy permutation if needed. We define a local greedy permutation to be the permutation of data that sorts the categorical values in ascending order for each dimension. We present a visual example of how this splitting criteria works in Appendix~\ref{GLP_example}. 
More formally, we define the minimum permuted negative log likelihood criteria as follows 
\begin{align}
    &\textnormal{MinPermNLL}(s,v) = \min_{Q,\pi} -\sum_{i=1}^\nobs \log Q_{\pi}(\xvec_i) \,,
\end{align}
where $Q_{\pi}(\xvec) \triangleq I(x_{s}\! \in \!v)Q(\xvec) + I(x_{s}\! \not\in \!v)Q(\pi(\xvec))$, $Q$ is a shared independent distribution, and the permutation for the split feature is the identity (no permutations), i.e., $\pi_s = \pi_{\Id}$ and $I$ is an indicator function that is 1 if the condition is true and 0 otherwise.
Because $Q$ is independent and $\pi$ is a feature-wise permutation, this split criterion can be trivially decomposed into $\ndim$ independent subproblems that can be solved exactly in 1D using sorting.
Concretely, without loss of generality, we assume that the counts on the left of the split are already in ascending order.
The optimal $\pi$ is merely the permutation such that the counts on the right are also sorted in ascending order, and the minimum likelihood $Q$ is equal to the empirical frequencies after sorting. 
We give the pseudocode of our ConstructTree algorithm and our FindBestSplit algorithm below in \autoref{alg:construct-tree} and \autoref{alg:find-best-split}, respectively. 
\begin{algorithm}[H]
\caption{ConstructTree: Learn node splits and count data at leaves} 
\label{alg:construct-tree} 
\begin{algorithmic}[1]
\REQUIRE Node $\node$, training data at node $\Xdata$, max depth $\maxdepthparam$, min samples to split $\minsplit$, split score function $\phi(\cdot, \cdot)$
\ENSURE Root node of decision tree structure $\node_0$
\IF{$|\node.X| \geq \minsplit$ \AND $\node.\text{depth} < \maxdepthparam$}
    \STATE $(\node.s, \node.v, \Xdata_{\lleft},\,\, \Xdata_{\rright}, \domain_{\lleft}, \domain_{\rright}) \gets \textnormal{FindBestSplit}(\Xdata, \node.\domain, \phi)$
    \STATE $\node.\lleft \gets \text{CreateNode}(\domain_{\lleft}, \node.\text{depth} + 1, \pi=\pi_{\text{Id}})$
    \STATE $\node.\rright \gets \text{CreateNode}(\domain_{\rright}, \node.\text{depth} + 1, \pi=\pi_{\text{Id}})$
    \STATE $\textnormal{ConstructTree}(\node.\lleft, \Xdata_\lleft)$
    \STATE $\textnormal{ConstructTree}(\node.\rright, \Xdata_\rright)$
\ELSE
    \STATE $\node.\counts^{\textnormal{init}} \gets \textnormal{CountsPerDimension}(\Xdata)$ \label{alg-line:init-counts}
\ENDIF
\STATE return $\node$
\end{algorithmic}
\end{algorithm}

\begin{algorithm}[H]
\caption{FindBestSplit: Find best split of data}
\label{alg:find-best-split} 
\begin{algorithmic}[1]
\REQUIRE Node data $\Xdata$, node domain $\domain$, split score function $\phi(\cdot,\cdot)$ that will depend on the splitting critera we use
\ENSURE Split feature $s$, split values $v$, left and right data $(\Xdata_{\lleft}, \Xdata_{\rright})$, left and right domains $(\domain_\lleft, \domain_\rright)$
\STATE $\mathcal{S} \gets \textnormal{GeneratePossibleSplits}(\domain)$
\STATE $s^*, v^* \gets \argmax_{(s,v) \in \mathcal{S}} \phi(\text{LeftSplit}(\Xdata, s, v), 
\newline
\text{RightSplit}(\Xdata, s, v))$
\STATE $(\Xdata_\lleft, \Xdata_\rright) \gets \text{Split}(\Xdata, s, v)$
\STATE $(\domain_\lleft, \domain_\rright) \gets \text{CreateChildDomains}(\domain, s, v)$
\STATE return $(s^*, v^*, \Xdata_\lleft, \Xdata_\rright, \domain_\lleft, \domain_\rright)$
\end{algorithmic}
\end{algorithm}
\subsection{Learning Node Permutations}
\label{learning_perm}

We first explore the theoretically optimal node permutations given the decision tree structure learned in the first step.
Then we will present our algorithm for learning node permutations and prove that it learns the theoretically optimal node permutations. 

\subsubsection{Theoretically Optimal Node Permutations}
Before we can provide our optimality theorem, we need to introduce several important definitions.
The next two definitions formalize the idea that the conditional distributions at each leaf node should be as close as possible to the marginal distribution over all leaf nodes.
\begin{restatable}[TSP Node Counts]{definition}{NodeCounts}
\label{def:node-counts}
Given a training dataset $\Xdata\in\zset^{\nobs\times\ndim}$, we define the TSP \emph{node counts} $c_{\node} \in \Z_+^{d \times k}$ for a node $\node$ as follows:
\begin{align}
    c_{\node}(j,a) &= \sum_{i=1}^\nobs \one(\sigma_{\tree}(\xvec_{i})_j = a \land \sigma_{\tree}(\xvec_i) \in \domain(\node))
\end{align}
where $j \in \{0,1,\dots,d-1\}$, $a \in \{0,1,\dots, k\}$, and $\one$ is an indicator function that is 1 if the condition is true and 0 otherwise.
\end{restatable}

\begin{restatable}[Rank Consistency]{definition}{RankConsistency}
\label{def:rank-consistency}
A TSP $\sigma_{\tree}$ is rank consistent if and only if there exists an independent permutation $\pi$ such that $\forall \node, \pi[c_{\node}] \in \rcset(\domain(\node))$,
where $\rcset(\domain(\node))$ is the set of \emph{rank consistent} count matrices $\rcset$ w.r.t. to a node's domain $\domain(\node)$ defined as
$\rcset(\domain) \triangleq \{c \in \Z_+^{d \times k}: \forall j, \,\, c(j) \in \rcset(\domain_j)\}\,,$
and $\rcset(\domain_j)$ is the set of rank consistent count vectors defined as:
$\rcset(\domain_j) \triangleq \{c(j) \in  \Z_+^k: \forall a<b, \in \domain, c(j,a) \leq c(j,b), \forall \ell \not\in \domain, c(j, \ell)=0 \} \,.$
\end{restatable}

We also need to formalize the notion that TSPs can have an equivalent decision tree structure even if they have different node permutations. 
\begin{restatable}[TSP Tree Equivalence]{definition}{TreeEquivalence}
\label{def:tree-equivalence}
Two TSPs $\sigma^{A}_{\tree}, \sigma^{B}_{\tree}$ are \emph{tree equivalent}, denoted by $\sigma^{A}_{\tree} \equivtree \sigma^{B}_{\tree}$ if and only if they have the same graph structure (i.e., same nodes and edges), and $\forall j, \xvec, \, \sigma_{\tree}^{A}(\xvec) \in \domain(\node_j^{A}) \Leftrightarrow \sigma_{\tree}^{B}(\xvec) \in \domain(\node_j^{B})$ where $j$ is the index of the node.
\end{restatable}
The proof that this is a true equivalence relation and more discussion is in the Appendix~\ref{tree_eq}. 
This means that any input value will traverse the same path through both TSPs and land at an equivalent leaf node in either structure.
Note that if the node permutations are different between the two TSPs, then the split values will also be different but equivalent. For example, if previously the node had an identity permutation and split on the value $v$, an equivalent node with permutation $\pi$ should split on $\pi(v)$. The tree equivalence definition ensures this is true for the whole tree and not just a single split.
Given these definitions, we can now provide our main optimality result.
\begin{restatable}[Optimality of rank consistent TSPs]{theorem}{OptimalityOfRankConsistentTsps}
\label{thm:optimality-of-rank-consistent-tsps}
Given a TSP tree $\tree^*$, rank consistent TSPs attain the optimal negative log-likelihood among TSPs that are tree equivalent, i.e., rank consistent TSPs are optimal solutions to the problem:
\begin{align}
    \argmin_{\sigma_{\tree} \in \Sigma(\tree^*)} \min_{Q_{\zvec}} -\frac{1}{n}\sum_{i=1}^n \log Q_{\zvec}(\sigma_{\tree}(\xvec_i)) \,,
\end{align}
where $\Sigma(\tree)$ is the set of TSPs whose trees are tree equivalent to $\tree^*$, and $Q_{\zvec}$ is an independent distribution.
\end{restatable}
The proof by contradiction is presented in Appendix~\ref{rank_proofs}.
Informally, if this property does not hold, there exists a pair of counts at a leaf node that could be switched and decrease the negative log likelihood leading to a contradiction.
Intuitively, this theorem states that making the conditional distributions after transformation of each leaf (i.e., local counts of each leaf node)  approximately match the marginal distributions (i.e., the global counts) is theoretically optimal---essentially the alignment of conditional distributions is the best way to get approximate independence.
This will lead to better log likelihood because we assume the prior distribution is independent.

\subsubsection{Two-Pass Algorithm to Learn Optimal Node Permutations}
\label{twopass}
Given the theoretical results of the previous section and the tree structure that we learned via one of our splitting criteria (an example of such tree is presented in the leftmost part of Figure~\ref{fig:passes}), we begin our two pass algorithm.
In the first pass (see \autoref{alg:learn-local-permutations}: LearnLocalPermutations), we traverse the tree in a bottom-up fashion to learn the local node permutations ($\modpi$), which are different from the node's final permutation $\pi$.
These auxiliary local node permutations $\modpi$ relate the local counts $\tilde{c}$ to the node counts of the new tree constructed in the next pass (this is formalized in 
\autoref{thm:relation-local-and-new-counts} in \autoref{alg_proofs}).
$\modpi$ is learned by sorting  the initial (unsorted) local node counts $\counts^{\init}$ in each dimension $j$ in ascending order only on the node's domain $\domain(\node)$.
Because the algorithm recursively does this from bottom to top, all local counts would be globally ranked (i.e., rank consistent).
We demonstrate this in the middle part of Figure~\ref{fig:passes} where at the leaf level we sort the counts and permute the data based on the sorting outcome, at the upper levels, we either simply combine the permuted counts from the level below (for all features except the split feature as seen on the left part of the tree) or combine the data, sort again and permute for the split feature (since combining can lead to a distortion in the sorted order, as seen in the right part of the tree).
However, this first pass does not construct a valid TSP.
So in the second pass (see \autoref{alg:construct-equivalent-tree}: ConstructEquivalentTree), we traverse the tree in a top-down fashion to construct an equivalent TSP (i.e., changing the node permutations $\pi_{\node}$ and split values $v$) given the auxiliary local permutations $\modpi$ learned in the first pass.
This new tree will be equivalent in tree structure to the original tree but now will be rank consistent and therefore optimal as per \autoref{thm:optimality-of-rank-consistent-tsps}.
This is demonstrated in the rightmost side of Figure~\ref{fig:passes}.
In this example, we see that if we were to start from the original data (the data at $N_0$ on the leftmost side of the figure) and apply the new TSP (with updated $\pi$ and $v$, we would arrive to the sorted data at the leaves, and our tree would be rank consistent.

\begin{algorithm}[H]
\caption{LearnLocalPermutations: Learn local permutations to sort local counts} 
\label{alg:learn-local-permutations} 
\begin{algorithmic}[1]
\REQUIRE Node $\node$ 
\ENSURE Modified node $\node$ with local permutations ($\modpi$)
\IF{$\node$ is a non-leaf internal node}
    \STATE $\node.\lleft \gets \textnormal{LearnLocalPermutations}(\node.\lleft)$
    \STATE $\node.\rright \gets \textnormal{LearnLocalPermutations}(\node.\rright)$
    \STATE $\node.\counts^{\init} \gets \node.\lleft.\counts + \node.\rright.\counts$ \label{alg-line:combine-counts}
\ENDIF
\STATE $\forall j, \node.\modpi_j \gets \textnormal{Sort1D}(\node.\counts_j^{\init}, \nodedomain_j)$
\STATE $\node.\counts \gets \text{PermuteCounts}(\node.\modpi, \node.\counts^{\init})$ \label{alg-line:sort-counts}
\STATE return $\node$
\end{algorithmic}
\end{algorithm}
\begin{algorithm}[H]
\caption{ConstructEquivalentTree: Construct a new tree with equivalent splits and permutations based on local permutations}
\label{alg:construct-equivalent-tree} 
\begin{algorithmic}[1]
\REQUIRE Node $\node$ with local permutations $\node.\modpi$, ancestor permutation $\pi_{anc}$
\ENSURE Modified node $\node$ with valid but equivalent permutation subtree structure
\STATE $\node.\pi \gets \pi_{anc} \circ \node.\modpi \circ \pi_{anc}^{-1}$
\IF{$\node$ is internal non-leaf node}
    \STATE $\node.v \gets \node.\modpi \circ \pi_{anc}(\node.v)$
    \STATE $\pi_{anc}^{\text{child}} \gets \node.\pi \circ \pi_{anc}$
    \STATE $\textnormal{ConstructEquivalentTree}(\node.\lleft, \pi_{anc}^{\text{child}})$
    \STATE $\textnormal{ConstructEquivalentTree}(\node.\rright, \pi_{anc}^{\text{child}})$
\ENDIF
\STATE return $\node$
\end{algorithmic}
\end{algorithm}

\begin{figure*}[htp]
\centering
\includegraphics[trim=0cm 0.7cm 0cm 2cm,width=.3\textwidth ]{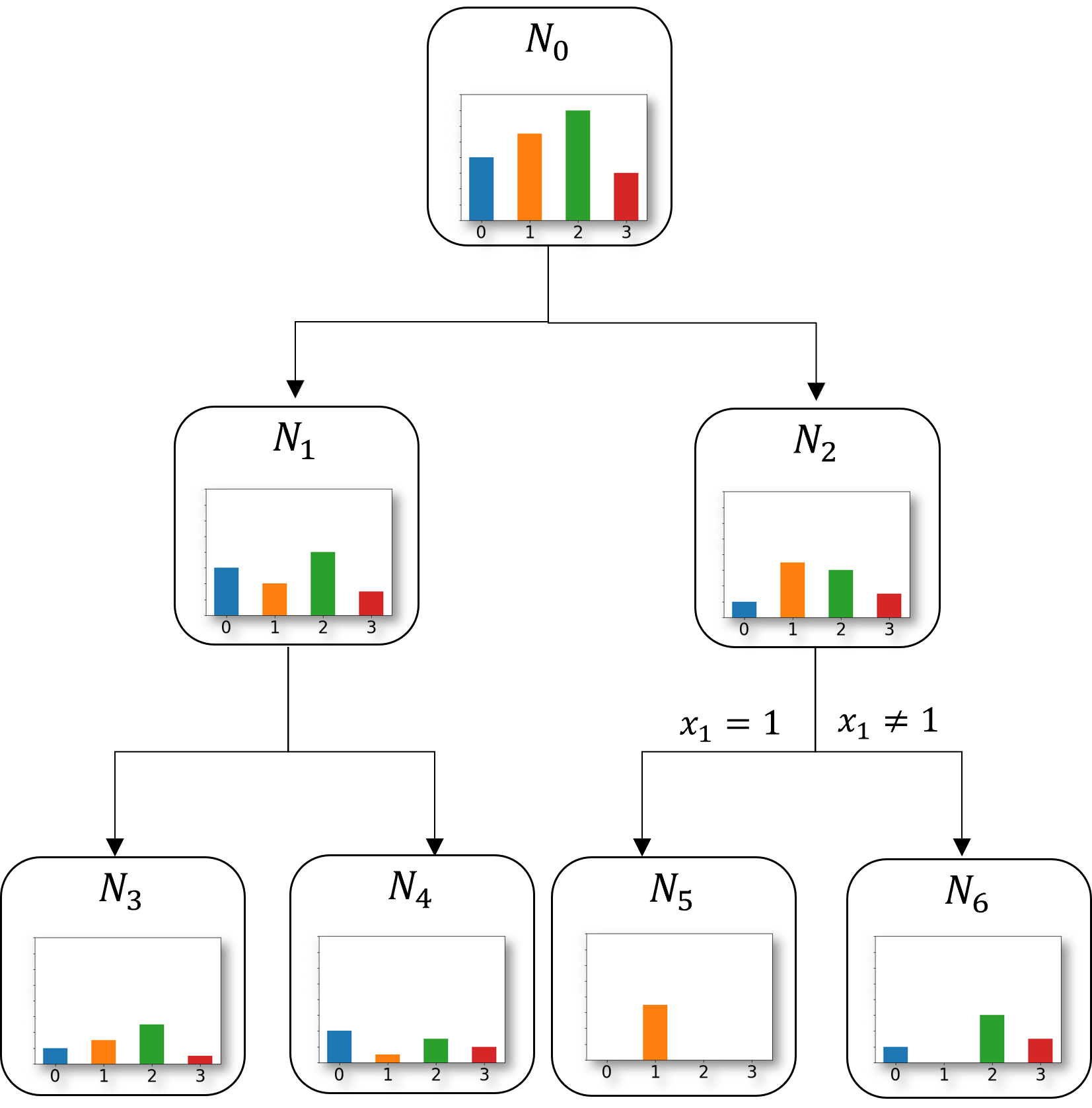}\hfill
\includegraphics[trim=0cm 0.7cm 0cm 2cm,width=.3\textwidth]{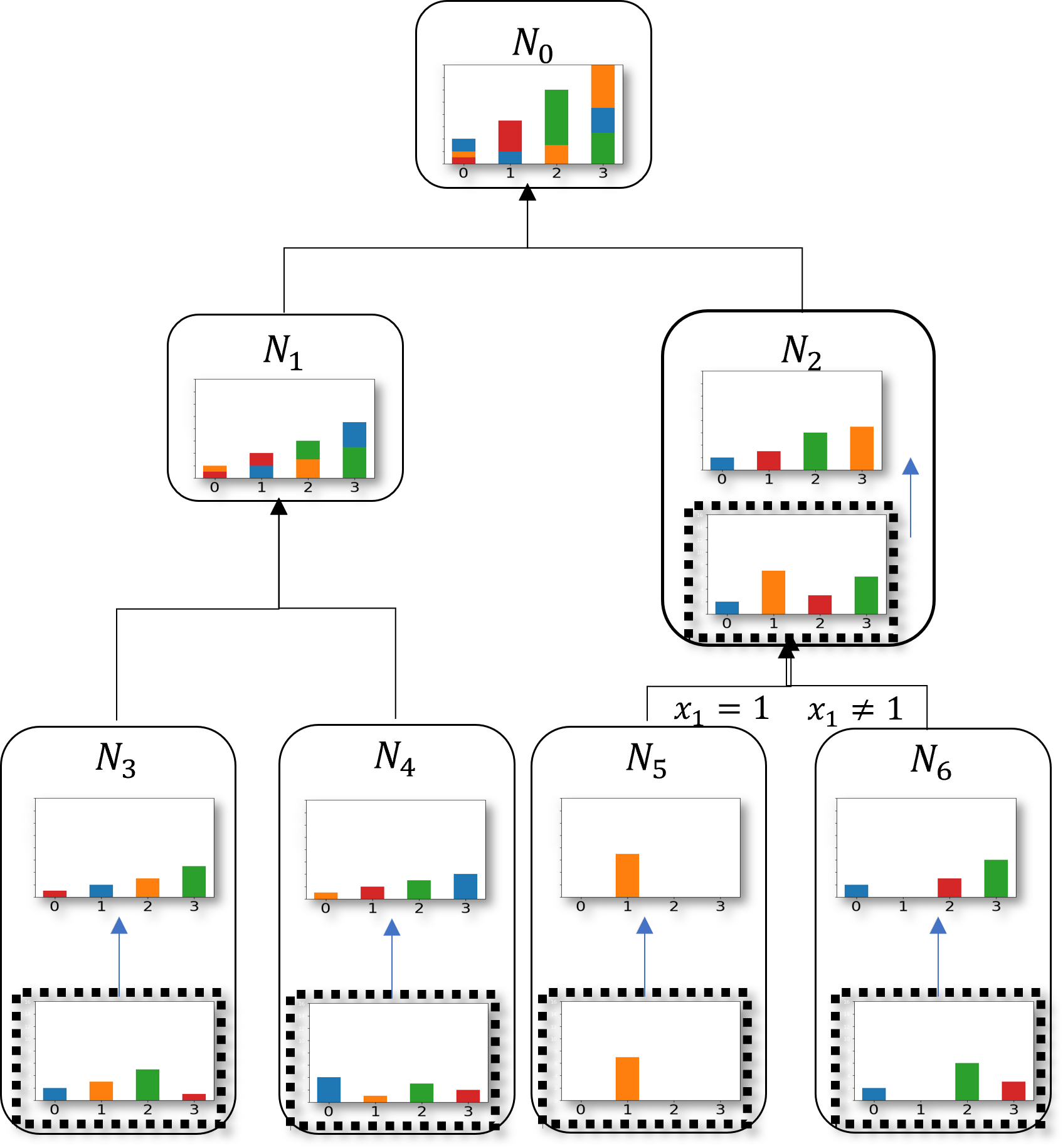}\hfill
\includegraphics[trim=0cm 0.7cm 0cm 0.5cm,width=.3\textwidth]{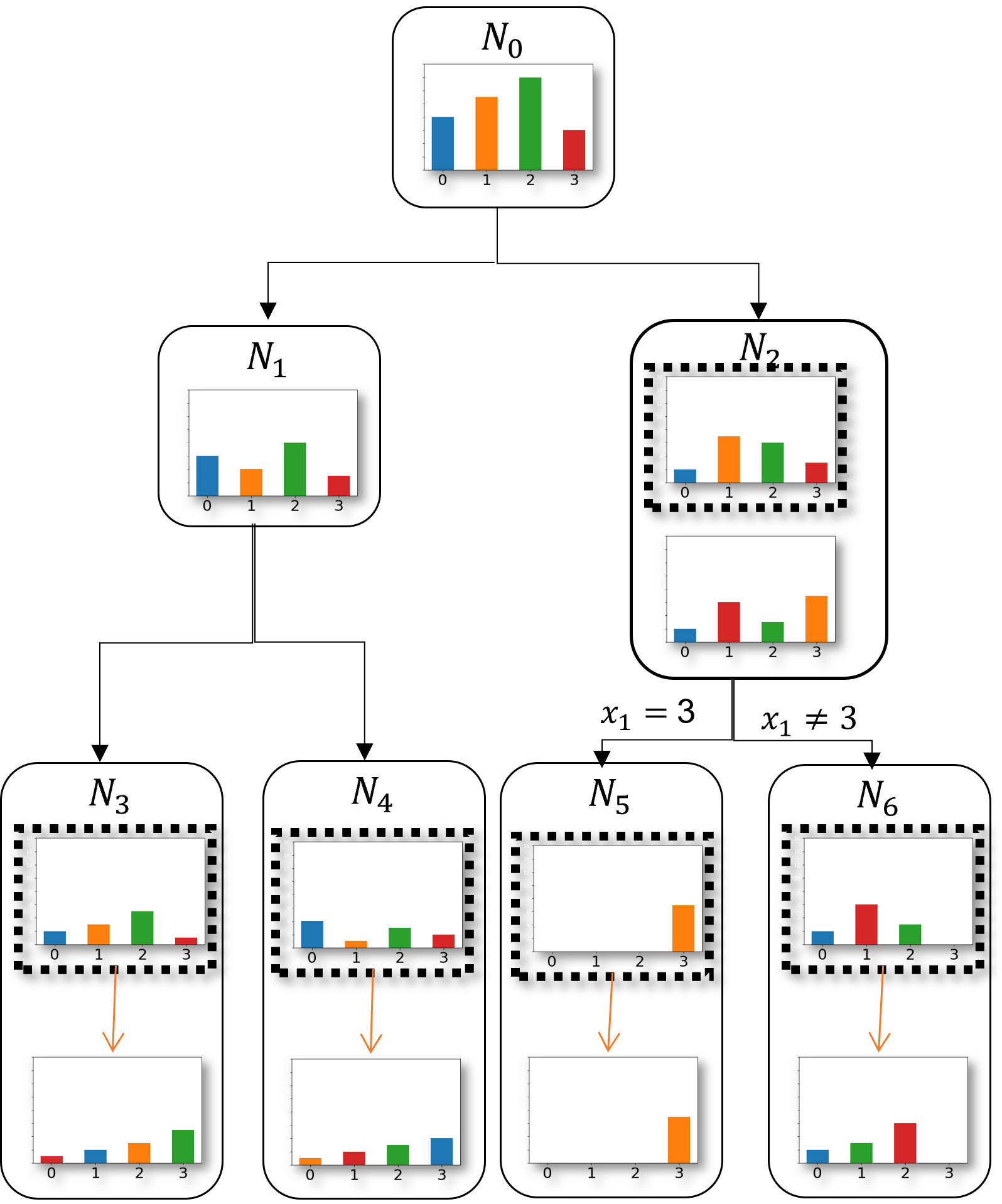}
\caption{On the left, we demonstrate tree construction. The bar plots in each node indicate the counts of the categories (0,1,2 and 3) for the feature ($x_1$), we show just $x_1$'s counts for simplicity, but the same idea applies to other features. The arrows are left blank when the data is split based on another feature that's not $x_1$. The different colors of the bars are to easily track the changes in counts when permutations are applied. In the middle, we demonstrate the operation of pass 1. The dashed lines around borders of the bar plots indicate that these are intermediate steps. The blue arrows indicate that we are learning and applying the local permutations for $x_1$ at the corresponding node. On the right, we demonstrate the resulting tree after both passes. The orange arrows indicate the new node permutations. Also, note how the split feature is now based on category 3 instead of 1 (i.e., we also updated $v$). Another more detailed example is presented in Appendix~\ref{twoPass}.}
\label{fig:passes}
\end{figure*}

We then prove two lemmas regarding this two-pass algorithm (proofs in \autoref{alg_proofs}).
The first lemma ensures that \autoref{alg:construct-equivalent-tree}: ConstructEquivalentTree produces a valid TSP that has equivalent tree structure.
The second lemma proves that the permutations learned in \autoref{alg:learn-local-permutations}: LearnLocalPermutations produce a rank consistent TSP. 
The proof of algorithm optimality follows easily from these two lemmas and \autoref{thm:optimality-of-rank-consistent-tsps}.
\begin{restatable}{lemma}{ConstructEquivalentTree}
\label{thm:construct-equivalent-tree}
The ConstructEquivalentTree algorithm produces a new TSP that has equivalent tree structure to the original TSP from the ConstructTree algorithm.
\end{restatable}
\begin{restatable}{lemma}{LearnLocalPermutations}
\label{thm:learn-local-permutations}
The local permutations learned in LearnLocalPermutations ensure that the equivalent TSP will be rank consistent.
\end{restatable}
\begin{restatable}[Algorithm finds optimal permutations]{theorem}{AlgorithmOptimal}
\label{thm:algorithm-optimal}
Given a tree structure $\tree$, our permutation learning algorithm finds the optimal permutations for the nodes in terms of negative log likelihood.
\end{restatable}

This result is particularly notable since the LearnLocalPermutations and ConstructEquivalentTree can be computed in closed-form using only sorting operations on count matrices and manipulations of independent permutations.
Thus, the overall computational complexity is linear (up to log factors) in all relevant parameters.

\section{Experiments}
To test our approach, we compare different results including the negative log likelihood (NLL),  training time (TT), and the number of model parameters (NP), against models for discrete flows that are suitable for categorical data and that allow for discrete latent spaces to be achieved by training by optimizing the exact likelihood. The models that fit these criteria are the two models that were introduced in \cite{tran2019discrete} and implemented in  \cite{bricken_trentbrickpytorchdiscreteflows_2021} (the Autoregressive Flow with Independent Base distribution (AF), and Bipartite Flow with independent base distribution (BF)) and the Discrete Denoising flows (DDF) that were introduced in \cite{lindt2021discrete}. 

Our model, Discrete Tree Flows (DTF), utilizes an independent base distribution, and has three variations depending on the splitting criteria that we use. $
\textnormal{DTF}_{GLP}, \textnormal{DTF}_{RND}$ is when our models uses the greedy local permutation, and random splitting critera respectively. To calculate the number of parameters in our DTF model, we counted the number of permutations that are not the identity in each node and also added a contribution of two from each node to account for the parameters of the split feature and split value that are encoded in the node. 
All the timing results we report are from running our experiments on an Intel\textregistered Core\textsuperscript{TM} i9-10920X 3.50GHz for CPU, and Nvidia GeForce RTX\textsuperscript{TM} 3090 for GPU. 
Note that we had to modify the implementation of the BF model in \cite{bricken_trentbrickpytorchdiscreteflows_2021} as it had some bugs, the details of those modifications are presented in Appendix~\ref{modify_bipartite}.

\subsection{Synthetic data}
We carry out a set of experiments on synthetic data and report the results for comparing our $\textnormal{DTF}_{GLP}$ model with AF, BF and DDF in \autoref{tab:syn_exp_results}. A full table that includes 
$\textnormal{DTF}_{RND}$ is presented in Appendix~\ref{app:syn_results_full}. In this table we report the results for the best performing model among those we have tried (We list all models we tried in Appendix~\ref{model_space} and the best models that we report the results for in Appendix~\ref{models_used}). The training times for the different models vary depending on the model's hyperparamteres, the size of the data $n\times d$ and the number of categories $k$.

\paragraph{Discretized Gaussian mixture model (8 Gaussian)}
Following the experiments in the \citep{tran2019discrete} paper that was also reproduced in \cite{lindt2021discrete}, we generate data from a mixture of 8 Gaussian distributions and then \emph{discretize} each feature into 91 categorical bins so that $d=2$ and $k=91$ dataset.
We choose to draw 12,800 samples in total, reserving 10,240 samples for training and 2,560 for testing.
Visualizations of samples trained on this discretized Gaussian mixture model dataset can be found in Appendix~\ref{visuals}.

\paragraph{Gaussian Copula Synthetic Data }
We also created synthetic data with pairwise dependencies using a Gaussian copula model with discrete marginals (as a reference for copula models see \citep{nelsen2007introduction}).
The generating process (summarized in Appendix~\ref{Copula}) is used to generate 3 datasets with varying degrees of dependency among the features.
For all distributions, we set $d=4$ and sampled $n = 10000$ and used an inverse CDF of a Bernoulli distribution to generate binary data with $k = 2$.
The first dataset has a very high correlation (COP-H) where the feature are very dependant, the second with moderate correlations (COP-M), the third with weak correlations (COP-W).

\paragraph{Synthetic Data Experimental Results:}

As can be observed from \autoref{tab:syn_exp_results}, Our method with the GLP splitting criteria outperforms all other methods in terms of NLL and training time (even when comparing training times on GPU for the other models), with the exception of the 8 Gaussian dataset where our model gives very comparable result to the best model (DDF) but at a fraction of the training time. 

\begin{table}[t]
\caption{Average of NLL and Training Time on CPU (TTC), and GPU (TTG) across 5 folds ($\pm$std) for synthetic datasets, the full table is available in Appendix~\ref{app:syn_results_full}}
\label{tab:syn_exp_results}
\begin{center}
\begin{small}
\begin{sc}
\begin{tabular}{p{0.038\linewidth}p{0.19\linewidth}p{0.19\linewidth}p{0.19\linewidth}p{0.18\linewidth}}
\toprule
 & AF  & BF & DDF & $\textnormal{DTF}_{GLP}$ \\
 
\midrule

\multicolumn{5}{c}{{\small \textbf{8Gaussian}}} \\
NLL &  6.92 \tiny{($\pm$ 0.06)} & 7.21 \tiny{($\pm$ 0.09)} & \textbf{6.42 \tiny{($\pm$ 0.03)}} &  6.5 \tiny{($\pm$ 0.03)}\\
TTC & 155.9 \tiny{($\pm$ 2.2)}  & 231.6 \tiny{($\pm$ 5.2)} & 119.8 \tiny{($\pm$ 0.8)} & \textbf{ 7.3 \tiny{($\pm$ 0.1)}}\\
TTG & 42.2 \tiny{($\pm$ 3.5)}  & 135.8 \tiny{($\pm$ 0.2)} & 79.7 \tiny{($\pm$ 0.8)} & NA\\

\multicolumn{5}{c}{{\small \textbf{COP-H}}} \\
$\textnormal{NLL}$ &  1.53 \tiny{($\pm$ 0.02)} & 1.47 \tiny{($\pm$ 0.06)} & 1.46 \tiny{($\pm$ 0.1)} & \textbf{1.33 \tiny{($\pm$ 0.02)}} \\
TTC  & 10.7 \tiny{($\pm$ 0.2)}  & 13.2 \tiny{($\pm$ 0.2)} & 58.1 \tiny{($\pm$ 1.0)} & \textbf{$\leq$0.1 \tiny{($\pm$ 0.0)}}\\
TTG & 14.6 \tiny{ ($\pm$ 0.4)}  & 20.9 \tiny{($\pm$0.3)} & 48.8 \tiny{($\pm$0.9)} & NA\\

\multicolumn{5}{c}{{\small \textbf{COP-M}}} \\
$\textnormal{NLL}$ &  1.76 \tiny{($\pm$ 0.1)} &  1.62 \tiny{($\pm$ 0.05)}& 1.51 \tiny{($\pm$ 0.16)} &  \textbf{1.4 \tiny{($\pm$ 0.02)}} \\
TTC  & 10.6 \tiny{($\pm$ 0.02)}  & 13.3 \tiny{($\pm$ 0.06)} & 77.9 \tiny{($\pm$ 1.8)} & \textbf{$\leq$0.1\tiny{($\pm$ 0.0)}}\\
TTG & 14.6 \tiny{($\pm$ 0.43)}  & 20.8 \tiny{($\pm$ 0.8)} & 66.9 \tiny{($\pm$0.5)} & NA\\

\multicolumn{5}{c}{{\small \textbf{COP-W}}} \\
$\textnormal{NLL}$ &  2.42 \tiny{($\pm$ 0.02)} & 2.35 \tiny{($\pm$ 0.03)} & 2.29 \tiny{($\pm$ 0.07)} & \textbf{2.22 \tiny{($\pm$ 0.02)}}\\
TTC  & 10.5 \tiny{($\pm$ 0.01)}  & 13.2 \tiny{($\pm$ 0.1)} & 77.3 \tiny{($\pm$ 1.7)} & \textbf{$\leq$0.1 \tiny{($\pm$ 0.0)}}\\
TTG & 13.9\tiny{($\pm$ 0.2)} & 19.2 \tiny{($\pm$ 0.2)} & 67.5\tiny{($\pm$  0.1)} & NA\\

\bottomrule
\end{tabular}
\end{sc}
\end{small}
\end{center}
\vspace{-2em}
\end{table}

\begin{table*}[ht]
\caption{Experiment results for real datasets. The lower the number the better. TTC, and TTG are the training times in seconds on CPU and GPU respectively.}
\label{tab:real_exp_results}
\centering
\begin{tabular}{p{0.18\linewidth}p{0.12\linewidth}p{0.12\linewidth}p{0.12\linewidth}p{0.12\linewidth}p{0.12\linewidth}}
\hline
& AF & BF & DDF & $\textnormal{DTF}_{GLP}$ & $\textnormal{DTF}_{RND}$\\
\hline

{\small \textbf{Mushroom Dataset}} & \multicolumn{5}{c}{}\\
$\textnormal{NLL}$ & 24.87 \tiny{($\pm$ 2.28)} & 23.02 \tiny{($\pm$ 2.3)} & 19.18 \tiny{($\pm$ 3.48)} & 
\textbf{14.15 \tiny{($\pm$ 2.44)}} & 16.66 \tiny{($\pm$ 2.98)}\\

$\textnormal{TTC}$ &  29.3 \tiny{($\pm$ 2.0)} & 20.9 \tiny{($\pm$ 2.7)} & 175.8 \tiny{($\pm$ 1.9)} & 
9.9 \tiny{($\pm$ 0.2)} & \textbf{0.5 \tiny{($\pm$ 0.0)}}\\

$\textnormal{TTG}$ & 7.7 \tiny{ ($\pm$ 1.0)}  & \textbf{6.0 \tiny{ ($\pm$ 1.3)}} & 75.2 \tiny{ ($\pm$ 1.0)} & 
NA & NA\\
$\textnormal{Number of Parameters}$ & 337952  & 522720 & 3290452 
& 7604 \tiny{($\pm$ 578)} & 13544 \tiny{($\pm$ 2352)}\\

\hline
{\small \textbf{MNIST Dataset}} & \multicolumn{5}{c}{}\\
$\textnormal{NLL}$ & 206.014 \tiny{($\pm$ 0.32)} &  205.94 \tiny{($\pm$ 0.26)} & \textbf{144.78 \tiny{($\pm$ 10.52)}} 
& 177.75 \tiny{($\pm$ 0.56)} & 187.44 \tiny{($\pm$ 1.17)}\\

$\textnormal{TTC}$ & 12104.6 \tiny{($\pm$ 359.2)} & 3290.5 \tiny{($\pm$ 13.3)} & 2909.3 \tiny{($\pm$ 45.4)} 
& 5213.7 \tiny{($\pm$ 204.9)} & \textbf{105.6 \tiny{($\pm$ 0.1)}}\\

$\textnormal{TTG}$ & \textbf{305.6 \tiny{ ($\pm$ 31.4)}}  & 308.7 \tiny{ ($\pm$ 4.1)} & 334.3 \tiny{ ($\pm$ 8.7)} & 
NA & NA\\

$\textnormal{Number of Parameters}$ & 44283456 & 42591578 & 2679408 & 
89583 \tiny{($\pm$ 2846)} & 19549 \tiny{($\pm$ 2061)}\\
\hline
{\small \textbf{Genetic Dataset}} & \multicolumn{5}{c}{}\\
$\textnormal{NLL}$ & 490.55 \tiny{($\pm$ 0.69)} & 471.54 \tiny{($\pm$ 1.87)} & 446.86 \tiny{($\pm$ 8.64)} 
& \textbf{437.19 \tiny{($\pm$ 1.02)}} & 470.9 \tiny{($\pm$ 6.1)}\\
$\textnormal{TTC}$ & 834.0\tiny{($\pm$ 2.1)} & 
251.6 \tiny{($\pm$ 0.5)} & 
209.4 \tiny{($\pm$ 0.6)} 
& 411.5 \tiny{($\pm$ 2.3)} & \textbf{5.9 \tiny{($\pm$ 0.0)}}\\

$\textnormal{TTG}$ & \textbf{23.8 \tiny{ ($\pm$ 0.9)}}  & 38.0 \tiny{ ($\pm$ 5.4)} &  29.5 \tiny{($\pm$1.1)} & 
NA & NA\\

$\textnormal{Number of Parameters}$ & 46686780 & 4484964 & 1205630 
& 9014 \tiny{($\pm$ 454)} & 14174 \tiny{($\pm$ 568)}\\
\hline

\end{tabular}
\vspace{-1em}
\end{table*}

\subsection{Real Data}
We follow the same settings used in the synthetic data experiments, and proceed to investigate the performance of our model on real datasets including some that are high dimensional (MNIST and Genetic).

\paragraph{Mushroom Dataset}
This dataset includes different attributes of mushrooms\footnote{\url{https://archive.ics.uci.edu/ml/datasets/Mushroom}},and has $n = 8,124$, $d = 22$ and the maximum number of categories in any column is $k=12$. 

\paragraph{Vectorized binary MNIST Dataset}
We then investigate the performance of our algorithm for the the vectorized and binarized MNIST dataset \cite{deng2012mnist}, where ``vectorized'' means that we do not leverage the image structure of the dataset but merely treat the data as a 784-dimensional vector.

\paragraph{Genetic Dataset}
Given that our algorithm is a general purpose discrete flow algorithm, we also investigate the performance of DTF on a realistic high-dimensional discrete genetic dataset to emphasize the generality of our method.
Specifically, we use a dataset of $n = 2504$ individuals with $d = 805$ sampled single nucleotide polymorphism (SNPs) that was made available by \cite{10.1371/journal.pgen.1009303}.
Where the SNPs were encoded as binary data ($k=2$).

\paragraph{Real Data Results}

We present results for the real-world datasets averaged across 3 folds for MNIST and Genetic data and across 5 folds for the mushroom data in Table~\ref{tab:real_exp_results}.
We notice some interesting results, our $\textnormal{DTF}_{RND}$ model always have the fastest training time, and its NLL results are either very comparable to the other model and sometimes even gives better results than the other models (it always exceeds the performance of AF and BF but sometimes falls behind DDF). This suggests that even with a very simple splitting criteria such as random splitting, our two-pass algorithm is efficient enough to learn interesting aspects about the data. Our $\textnormal{DTF}_{GLP}$ usually gives the best NLL results (with the exception of MNIST), but tends to take more training time, as finding the best splits dominate the training time in this case. This was not an issue with lower dimensional datasets (like the synthetic data), but becomes more prominent when the dimension of the data increases (as in MNIST and Genetic), especially when comparing our model's CPU training times to the other model's GPU training times. 
The only case where our model doesn't seem in par with the best NLL result is for MNIST. While our model gives better results than AF and BF, it fails to out-preform DDF. This suggests that our model is better suited for tabular data, unlike MNIST which presume a more dimension-by-dimension dependent vectorized dataset. Our models also has a smaller number of parameters, and they're not a fixed set of parameters as with the other models, our parameters can grow depending on the data and what permutations are deemed necessary, which makes it more flexible than a typical flow model.

\section{Discussion and Conclusion}

We presented a novel framework for discrete normalizing flows called DTF that relies on tree-structured permutations (TSPs), which we define and develop. We prove that our learning algorithm finds the optimal node permutations given a decision tree structure.
Empirically, our model results demonstrate that DTF outperforms prior approaches in terms of NLL while being substantially faster for most experiments.

We note that our model does have some limitations. First, while we do guarantee permutations optimality given the tree structure, the tree structure itself is not guaranteed to be optimal since the tree is grown greedily as with most decision tree algorithms. Moreover, our splitting is required to be axis aligned in our current framework and thus cannot perform complex split operations. We believe this could be addressed by constructing coupling-like TSPs, similar to coupling layers in continuous normalizing flows. Specifically, the split functions could be arbitrarily complex functions of half of the features while the node permutations only permute the other half of the features. This would allow deep models to be used for the split functions.
This idea of coupling-like TSP can also address another issue of handling very high dimensional data (e.g., $d>1000$) because our non random split algorithms are naively $O(d^2)$.
This can also be handled in our current approach by drawing from techniques used in decision tree algorithms.
Second, similar to other methods, large values of $k$ may be challenging for the split criteria.
For random splitting, this should be straightforward by selecting a random set of categorical values to go to the left (possibly based on counts for each category).
For GLP, we implemented the simplest case of choosing only one value to go left.
However, we could greedily select the next best feature to go to the left but this would increase the computational complexity by the max number of values that go to the left.
Third, we have focused on tabular categorical discrete data (even MNIST is vectorized and treated as a 784-dimensional vector). Extending to discrete image data  or text-based data is non-trivial.
Prior gradient-based works such as AF and DDF were able to leverage CNNs or NLP-based transformer architectures.
Thus, we do not expect our current tabular-focused DTF approach to be competitive on these tasks.
As we mentioned earlier, it may be possible to extend our framework to use arbitrary NNs for the split function if we partition the features into fixed and free sets (as is done in standard flow coupling layers).
Or, for NLP-based applications, our framework may be extended to have an autoregressive-like structure to our TSPs.
Ultimately, we hope that our paper lays the groundwork for developing practical and effective discrete flows using decision tree algorithms.

\section*{Acknowledgements}
All authors acknowledge support from the Army Research Lab through contract number W911NF-2020-221.

\bibliography{ref}
\bibliographystyle{icml2022}

\newpage
\appendix
\onecolumn

\onecolumn
\section{Overview}
We have organized our appendix as follows:
\begin{itemize}
    \item \textbf{Appendix~\ref{GLP_example}} provides a visual explanation of the GLP splitting criteria.
    \item \textbf{Appendix~\ref{twoPass}} provides a visual example of our two pass algorithm.
    \item \textbf{Appendix~\ref{universal_proof}} describes the expressivity of DTF models by proving they are universal permutation models.
    \item \textbf{Appendix~\ref{inv_proofs}} ultimately proves invertibility of TSPs by including a proof of \cref{lem:recoverability-of-tree-path}, proof of \cref{thm:invertibility-constraint}, and a necessary and sufficient condition for invertibility with its proof.
    \item \textbf{Appendix~\ref{rank_proofs}} proves the optimality of rank consistent of TSPs with the proof of \cref{thm:optimality-of-rank-consistent-tsps}.
    \item \textbf{Appendix~\ref{tree_eq}} proves that  \cref{def:tree-equivalence} is an equivalence relation.
    \item \textbf{Appendix~\ref{code}} provides pseudo code for all algorithms.
    \item \textbf{Appendix~\ref{alg_proofs}} includes proofs of the algorithms by proving \autoref{thm:algorithm-optimal} with multiple lemmas and definitions.
    \item \textbf{Appendix~\ref{exp_details}} provides more experimental details and results (i.e. modification and details of codes and architecture, extra details of dataset configuration, and more table results and figures).
\end{itemize}
We give a reference table of the notations used throughout the paper and appendix below on \autoref{table:notation}:
\clearpage

\begin{table}[!ht]
    \caption{Notation}
    \label{table:notation}
    \begin{tabularx}{\textwidth}{cX}
        \toprule
        \multicolumn{2}{l}{{\underline{indices:}}}\\
        $\xvec$ & input data \\
        $d$ & number of dimension/features \\
        $n$ & number of samples \\
        $k$ & number of categories \\
        $s$ & split feature \\
        $j$ & denotes feature that usually excludes the split feature (e.g. $j \neq s$) \\
        $v$ & split value \\
        $\tree$ & tree \\
        $\node$ & node \\
        $\pi$ & independent permutation usually on a specific node \\
        $\sigma_{\tree}$ & a tree-structured permutation based on tree $\tree$ \\
        $\treepath$ & tree traversal path of a specific input $\xvec$ \\
        $\domain$ & categorical domain (e.g. $\domain(\node)$ would denote all categorical domain of the specific node) \\
        $Q_{\zvec}$ & independent base distribution \\
        $\Sigma$ & an exhaustive set of equivalent trees (e.g. $\Sigma(\tree)$ would be set of all TSPs that are equivalent to $\tree$) \\
        $c(j,a)$ & count value of feature $j$ for categorical value $a$ usually for a specific node (can be represented by a matrix) \\
        $\rcset$ & denotes rank consistency for count matrices w.r.t. to a node domain (e.g. $\rcset(\domain(\node))$)\\
        \\
        \multicolumn{2}{l}{{\underline{subscript:}}}\\
        $[\cdot]$ & the level on the tree usually a subscript for $\node$ (e.g. $\node_{[n]}$ would represent the nodes on level $n$ of the tree) \\
        $(\cdot)$ & the max depth of tree usually a subscript for $\tree$ (e.g. $\tree_{(n)}$ would represent a tree with max depth $n$) \\
        $\lleft$ & represents the left child (e.g. $\node_{\lleft}$ would denote the left child node of that specific node) \\
        $\rright$ & represents the right child (e.g. $\node_{\rright}$ would denote the right child node of that specific node) \\
        $\anc$ & denotes all the ancestral nodes on the current node and is usually a subscript for $\pi$ (see \autoref{eqn:anc_def}, \autoref{eqn:inverse_anc_def}, \autoref{eqn:general_new_ancestor_perm}) \\
        $\textnormal{leaf}$ & the leaf nodes (e.g. $\tree_{\textnormal{leaf}}$ would be the leaf nodes of this specific tree) \\
        \\
        \multicolumn{2}{l}{{\underline{superscript:}}}\\
        new & represents the new tree constructed after Alg~\ref{alg:construct-equivalent-tree} (e.g. $\pi^{\new}$ represents new tree permutation and $c^{\new}$ represents new counts)\\ 
        init & represents unsorted local counts only used with the local counts notation (e.g. $c^{\textnormal{init}}$)\\
        \\
        \multicolumn{2}{l}{{\underline{others:}}}\\
        $\pi[\cdot]$ & applies permutation to the values of vectors (i.e. $\pi[c]\triangleq[\pi_0[c(0)], \pi_1[c(1)], \cdots, \pi_d[c(d)]]$)\\ 
        $\tilde{\cdot}$ & represents ``local'' auxiliary objects (e.g. $\tilde{\pi}$ represents local permutation and $\tilde{c}$ represents local counts at each node computed in Alg~\ref{alg:construct-tree} and \ref{alg:learn-local-permutations}) \\
        $\equivtree$ & represents equivalence relation between two sets in this case we use it for trees (e.g. $\sigma_{\tree}^A \equivtree \sigma_{\tree}^B$) \\
        \bottomrule
    \end{tabularx}
\end{table}

\section{A visual explanation of the GLP splitting criteria}
\label{GLP_example}
\begin{figure}[H]
\centering
\includegraphics[width=\linewidth]{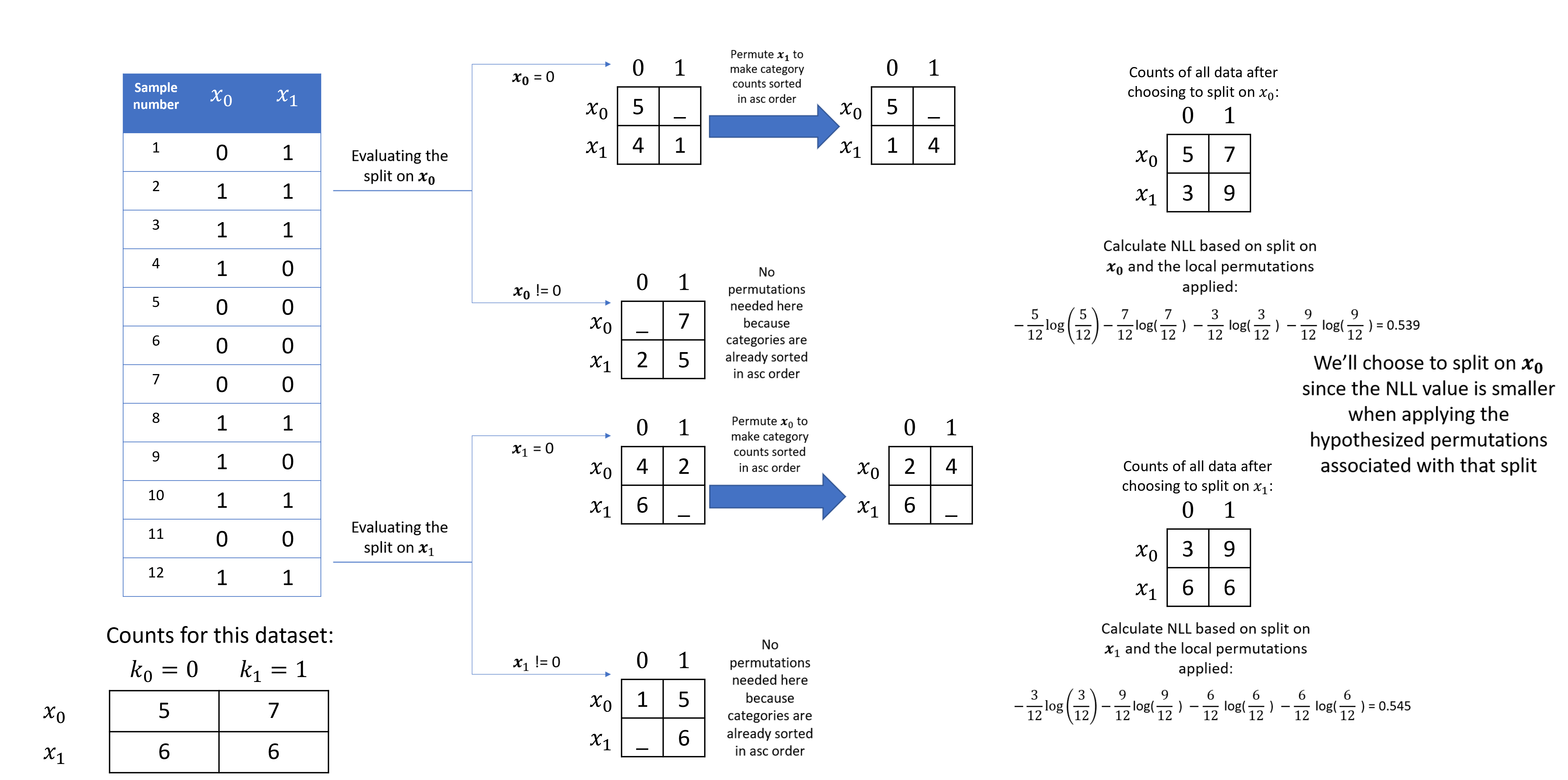}
\caption{An example for how the GLP criteria is evaluated.
}
\label{fig:glp_example}
\end{figure}

\newpage

\section{A visual example of the two pass algorithm}
\label{twoPass}

\label{example}

\begin{figure}[H]
\centering
\includegraphics[width=0.6\linewidth]{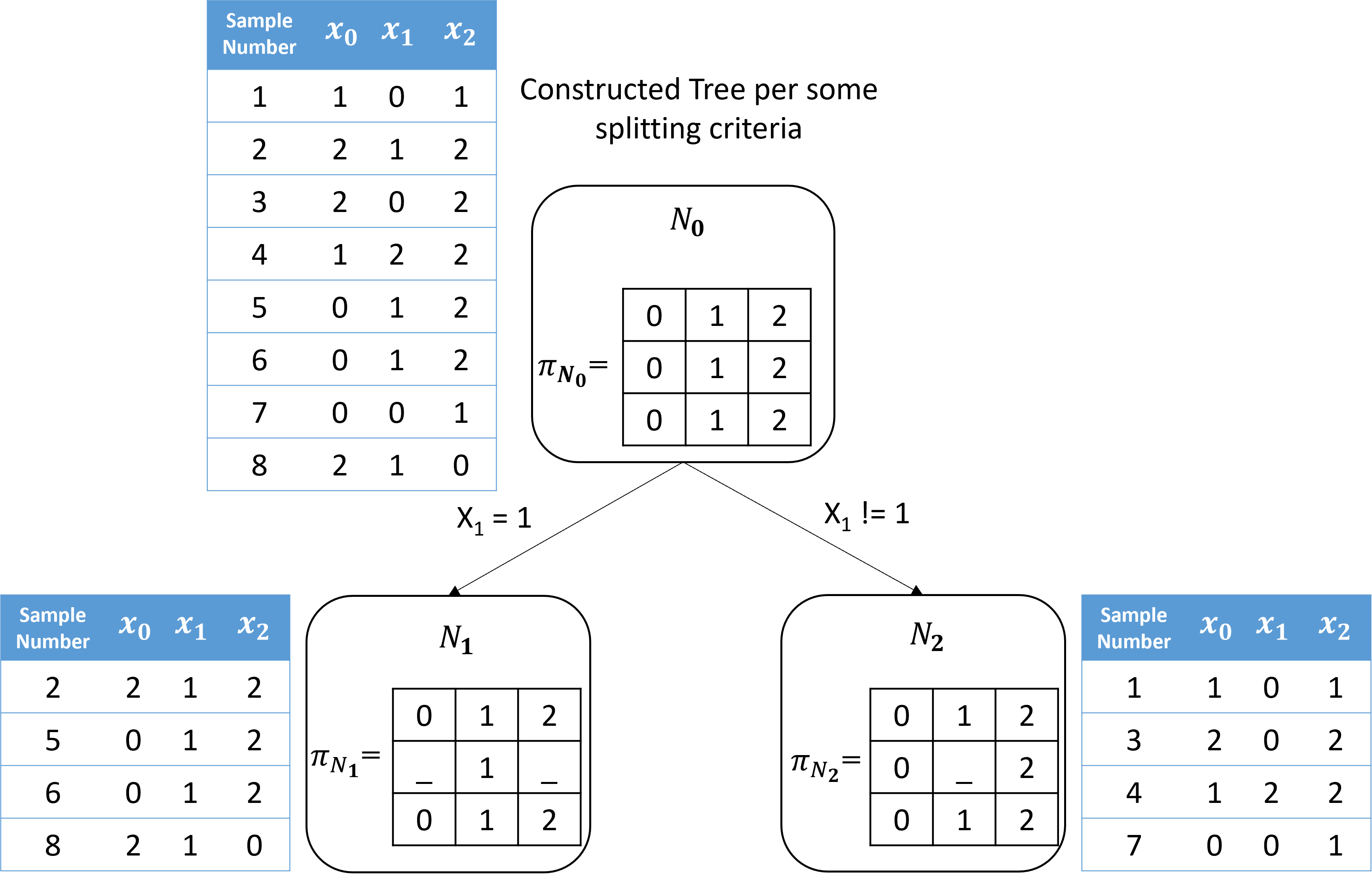}
\caption{An example of the tree constructed using some splitting criteria of choice, the permutation matrices are just the identity since we only learn the split information here.
}
\label{fig:pass1}
\end{figure}

\begin{figure}[H]
\centering
\includegraphics[width=0.6\linewidth]{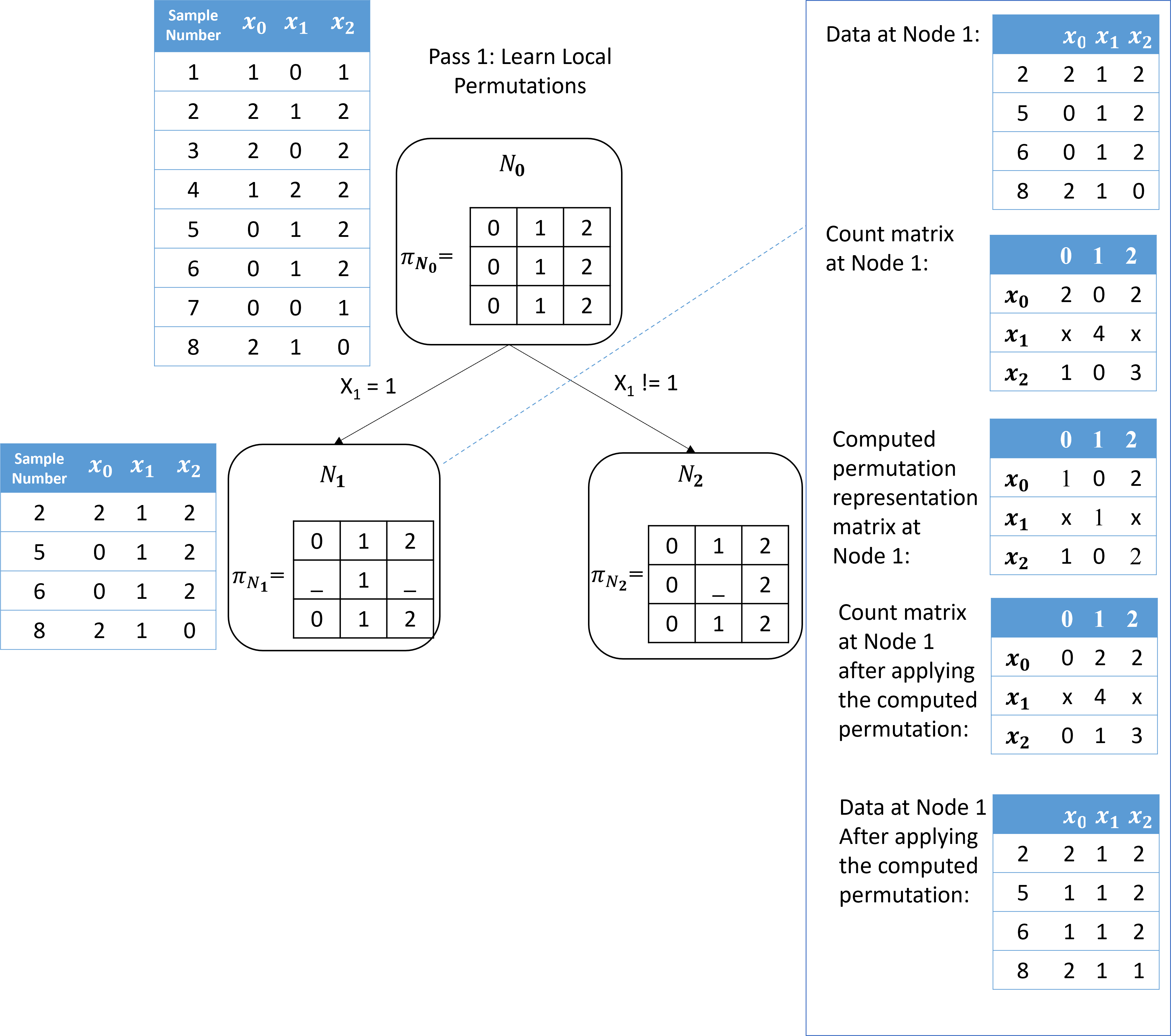}
\caption{Following tree construction, we start pass 1 at the bottom of the tree, the figure shows the steps done at node 1 in full details until the permutation matrix is calculated.
}
\label{fig:pass2a}
\end{figure}

\begin{figure}[H]
\centering
\includegraphics[width=0.6\linewidth]{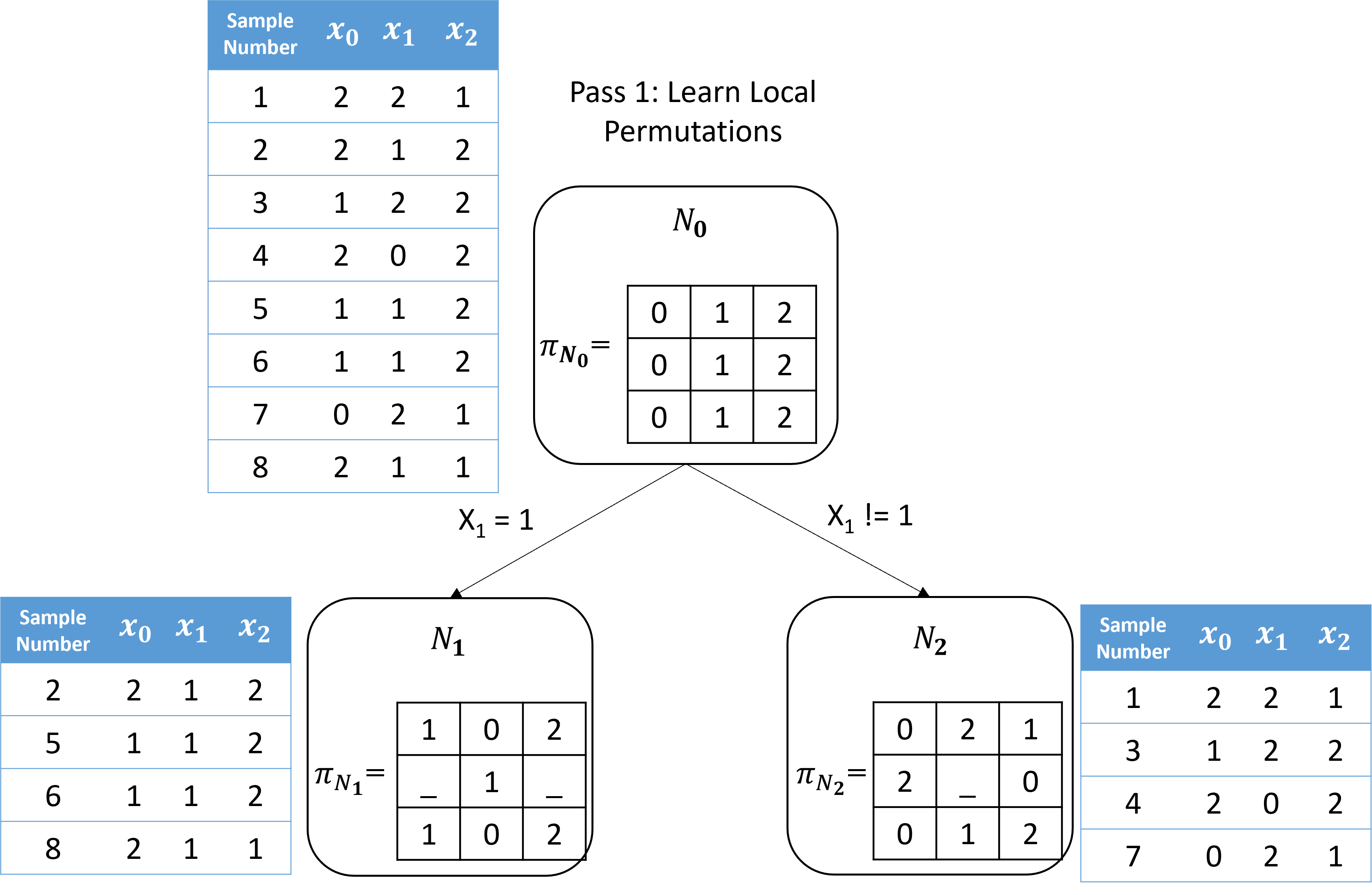}
\caption{The same logic is applied to Node 2, and we propagate the new permuted data up to node 0.
}
\label{fig:pass2b}
\end{figure}

\begin{figure}[H]
\centering
\includegraphics[width=0.6\linewidth]{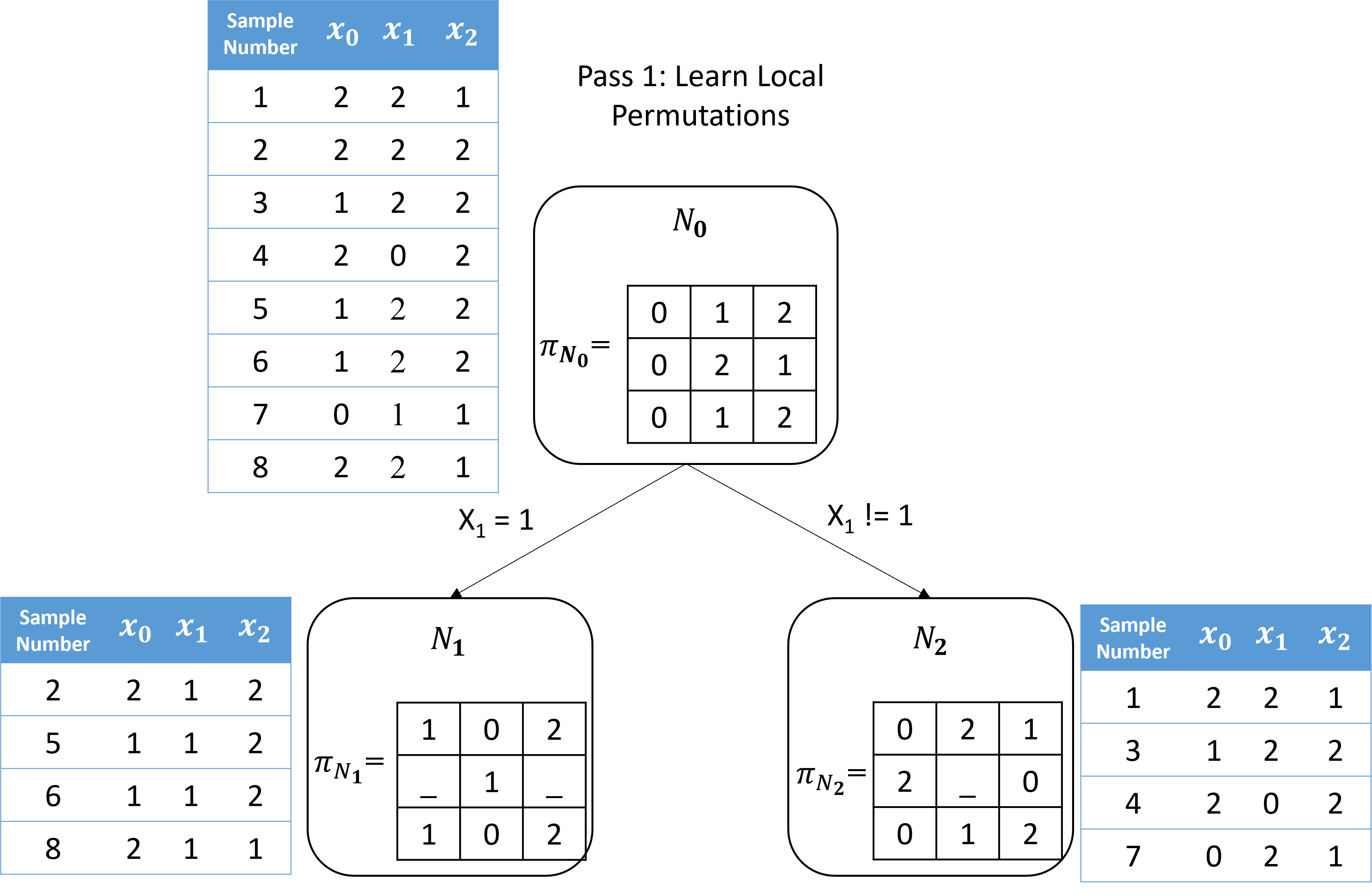}
\caption{At node 0, we use the same logic to compute the permutation matrix at that node -- notice that only the split feature can have a permutation since the others will be already in order. 
}
\label{fig:pass2c}
\end{figure}
\newpage
\begin{figure}[H]
\centering
\includegraphics[width=0.6\linewidth]{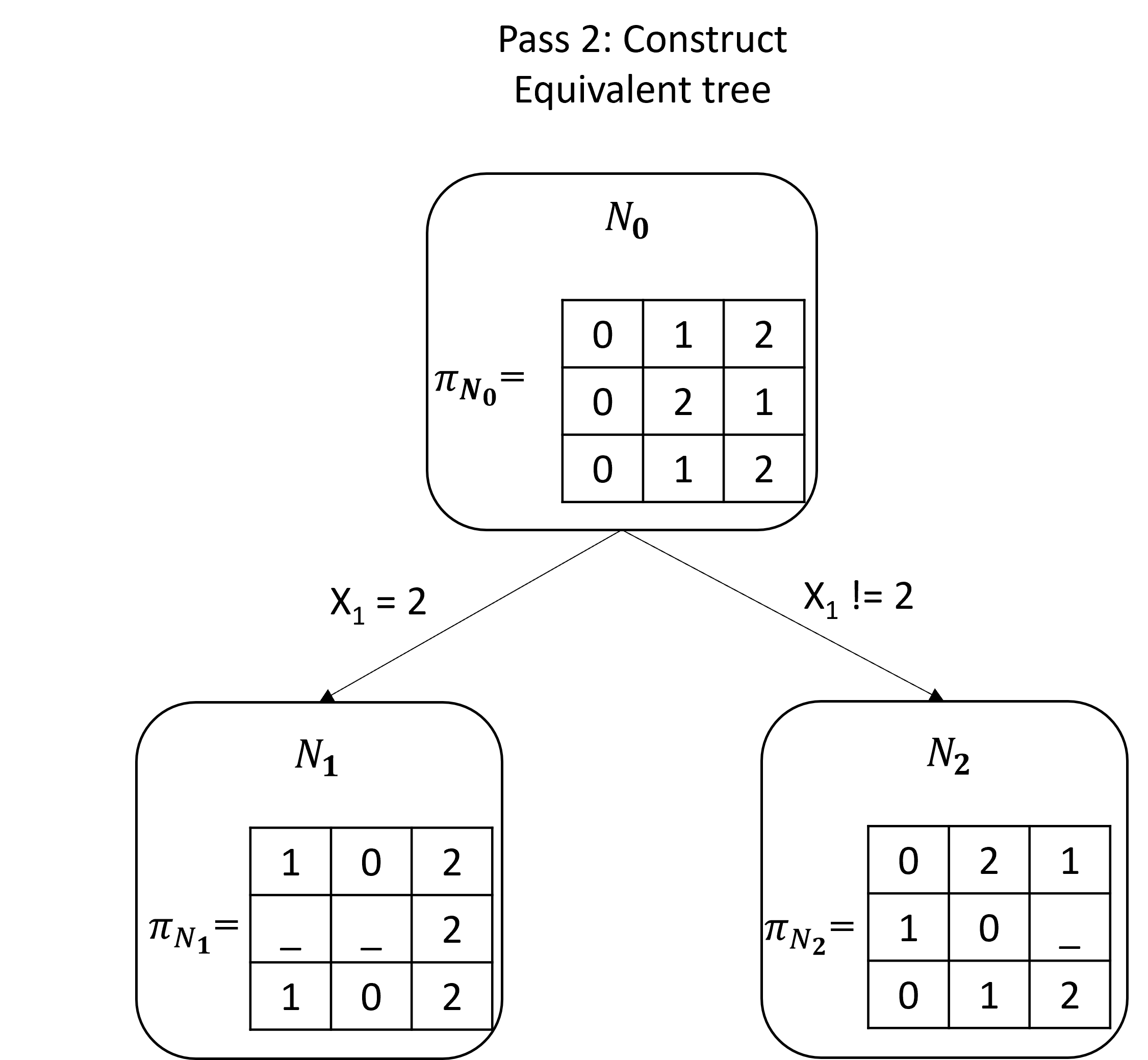}
\caption{In pass 2, we traverse the tree from top to bottom to fix it, "fixing it" includes applying the ancestor's node permutations to the split values, and to the permutation matrices as well of the children nodes. 
}
\label{fig:pass2c}
\end{figure}

\section{Proof of Expressivity of DTF}
\label{universal_proof}
We will prove that our DTF (i.e. a sequential combination of TSPs) can theoretically produce a universal permutation (i.e. an $n$ element permutation is universal if it contains all permutations of length $n$) on its domain. To prove this universality we must demonstrate that our DTF is a generating set of the symmetric group $S_n$ where $S_n$ is an exhaustive collection of permutation subsets of an $n$ elements set. $n$ is the total number of elements in our sample space (i.e. $n=144$ for a 2 feature 12 category domain, and $S_n$ would have a total of $n!$ subsets).

Our proof will rely on the theorem below:
\begin{theorem}[from \cite{pdf:gen_set}]
    \label{theorem:generatingset_adj}
    For $n\geq 2$, $S_n$ is generated by the $n-1$ transpositions/bijective permutations in this cyclic notation form $(1\; 2), (2\; 3),...,(n-1\; n)$.
\end{theorem}

\begin{definition}
Given a configuration $\xvec$ and a new value for the $j$-th feature denoted $\tilde{x}_j \neq x_j$, a \emph{single feature swap} is a permutation that permutes $\xvec = x_1 x_2 \cdots x_d$ and $\tilde{\xvec} = x_1 x_2 \cdots \tilde{x}_j \cdots x_d$, $\forall j$ $\exists x_j\in \nodedomain$, while keeping all other configurations the same, which could be denoted in cyclic notation as $(\xvec, \tilde{\xvec})$.
\end{definition}

We will show that a combination of DTFs can express each adjacent transposition corresponding to \autoref{theorem:generatingset_adj}. Let us first illustrate that a single TSP can implement a single feature swap (e.g. $({15342}_7\; {15344}_7)$ where the element has a base of 7 and only a single category on the fifth feature permutes $2 \leftrightarrow 4$ and the other features and categories do not change). 

Suppose there is a TSP with a max depth equal to $d$ where $d$ is the total number of features. Let us focus on the left most non-leaf nodes on each level of the tree where each left node fixes a single category on the non-permuting feature of our desired permuting configuration and let all node permutations by default be an identity. Next, if we store the permuting categories on the second to the leftmost leaf node, we can apply the desired permutation on this leaf node. Thus, obtaining a single feature swap. An example of this can be seen below.

\begin{figure}[H]   
    \centering
    \includegraphics[width=0.6\linewidth]{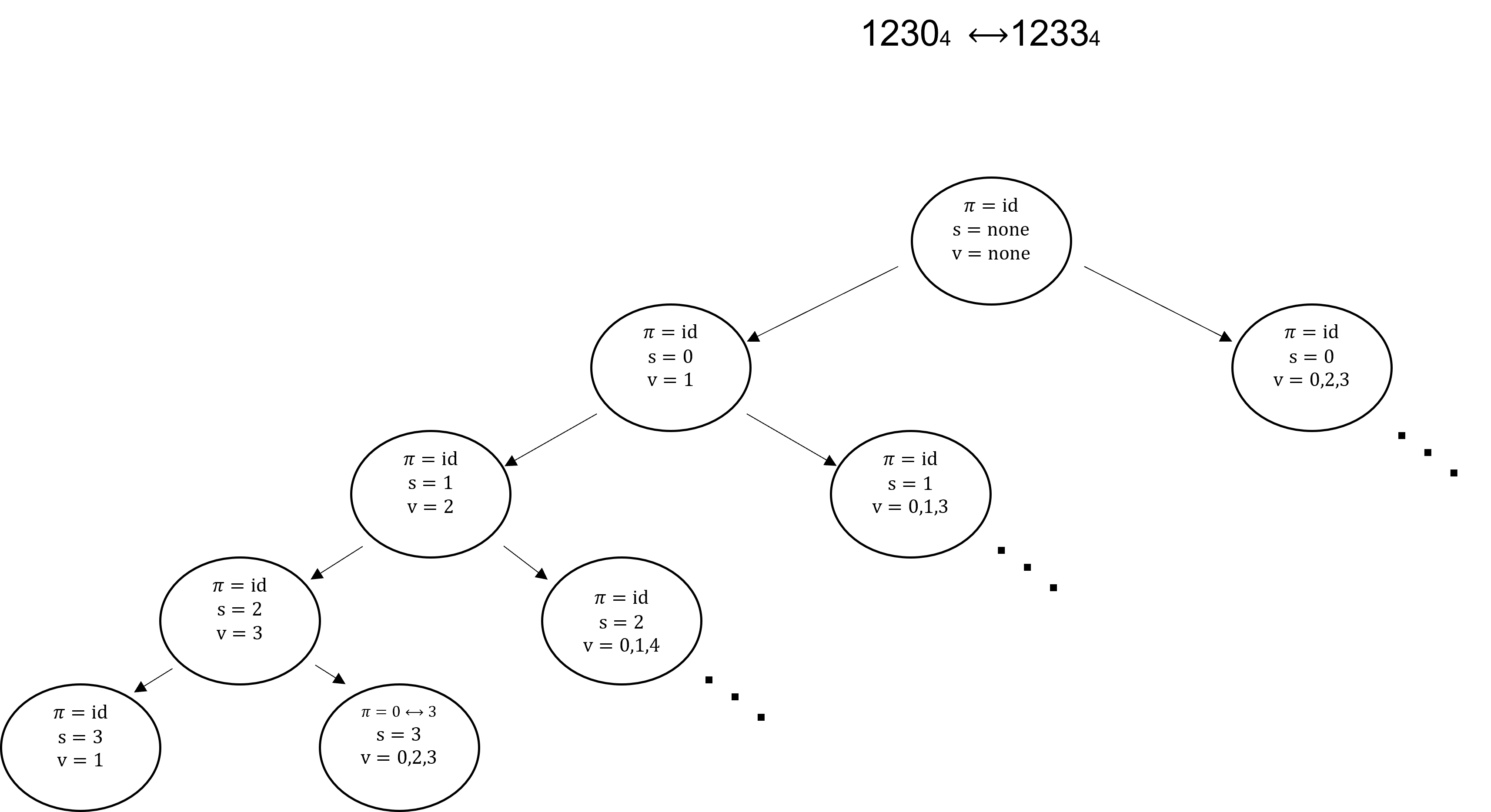}
    \caption{An example of a single feature swap.}
    \label{fig:ex_of_SFPC_perm}
\end{figure}
As seen in \autoref{fig:ex_of_SFPC_perm}, $1230_4 \leftrightarrow 1233_4$ has 4 features, therefore we have a TSP of max depth 5. The non-permuting features are the 0, 1, 2 features, so we fix the categorical values of these features on the leftmost none-leaf nodes. We want to permute on feature 3, thus we store these categorical on the 2nd leaf node and apply our single feature swap on this node (i.e. $0\leftrightarrow 3$).

Above, we demonstrated that we can permute any pair of categories on a single feature and restrict the others. We will use this condition to show that a sequential combination of single feature swaps can traverse the whole permutation space by showing that there is a snake-like path  that can fill the permutation space.

\begin{lemma}
\label{lem:snake-like_path}
There exists a snake-like path that reaches all discrete configurations where adjacent configurations on the path differ in only one feature value.
\end{lemma}

\begin{proof}
We will prove \cref{lem:snake-like_path} by induction.
   We will prove that a n-dimensional discrete space can be reduced to have the same cardinality as a single dimensional discrete space (i.e. $\mathbb{Z}_+^n \rightarrow \mathbb{Z}$). Thus, we show that there is a full adjacent single feature varying path that can traverse the whole n-dimensional space.

\paragraph{Base case $(n=1)$}
    Let the space have a base number of $b+1$ where $b\geq 0$. Obviously, We can use a traversal path of $0\leftrightarrow1\leftrightarrow2\leftrightarrow\cdots\leftrightarrow b$ for the single dimension, thus proving there is a 1-D discrete path.

\paragraph{Induction step$(n=k\Rightarrow n=k+1)$}
    Let us assume for up to a $k$-dimensional permutation space with a base number of $b+1$, the $k$-dimensional discrete space can be traversed with a 1-D discrete path. Now, let us look at the $k+1$ dimensional discrete space. We know that this traversal path is valid $0\underbrace{0\cdots0}_k \leftrightarrow \cdots \leftrightarrow 0\underbrace{b\cdots b}_k$ with our assumption that the $k$-dimensional space can be traversed with a 1-D discrete path. Rather than traversing the path $0\underbrace{b\cdots b}_k \leftrightarrow \cdots \leftrightarrow 1\underbrace{0\cdots 0}_k$, we traverse the path $0\underbrace{b\cdots b}_k \leftrightarrow 1\underbrace{b\cdots b}_k$. Again, we know that we can traverse the path $1\underbrace{b\cdots b}_k \leftrightarrow \cdots \leftrightarrow 1\underbrace{0\cdots 0}_k$, and thus we can traverse the path $1\underbrace{0\cdots 0}_k \leftrightarrow 2\underbrace{0\cdots 0}_k$. A repetition of this traversal process shows us that the $k+1$ discrete permutation space has a snake-like 1-D path traversal (i.e. $\underbrace{0\cdots 0}_{k+1} \leftrightarrow \cdots \leftrightarrow \underbrace{b\cdots b}_{k+1}$). Note, this is just one path of many that can be taken to traverse the space.
    
    Therefore, there exists a snake-like path that traverses the whole discrete permutation space. 
\end{proof}

Now by using the snake-like path from \autoref{lem:snake-like_path} we can put all the configurations into an 1D sequence of categorical values (which has $k^d$ unique values).
And from above, we know that a single TSP can swap any two of these adjacent categorical values.
Thus, these TSPs that swap any two adjacent values is a generating set for all possible permutations of the $k^d$ configurations by \autoref{theorem:generatingset_adj}.
Therefore, a composition of these TSPs can express any possible permutation (albeit possibly an exponentially large number of TSPs) and thus our TSPs with independent node permutations do not hinder the expressivity of DTFs.

\section{Proofs for invertibility of TSPs}
\label{inv_proofs}
As a helpful idea, we first define the range of a node to simplify the proofs in certain cases.

\begin{definition}[Range of a Node]
We define the \emph{range of a node}, denoted $\mathcal{R}(\node)$, as the image of $\mathcal{D}(\node)$ under $\pi_{\node}$, i.e., $\mathcal{R}(\node) \triangleq \{\pi_{\node}(\xvec) : \xvec \in \mathcal{D}(\node) \}$.
\end{definition}

\subsection{Proof of \cref{lem:recoverability-of-tree-path}}

\RecoverabilityOfTreePath*

\begin{proof}

First, note that because $\pi$ is a permutation (i.e., one-to-one mapping) and $\pi_{\node}(x) = x, \forall x \not\in \nodedomain$ (ie. configurations outside the domain will remain unchanged), then $\forall \xvec \in \mathcal{D}(\node), \pi_{\node}(x) \in \mathcal{D}(\node)$, i.e., all node permutations $\pi_{\node}$ do not permute configurations from in the domain to outside the domain.

Without loss of generality, we consider trees where all leaf nodes are at the max depth of $M$.

We will denote the tree traversal path of an input $\xvec$ up to tree level $i$ to be $\treepath_{(i)}(\xvec) \triangleq (\node_{[0],\xvec}, \node_{[1],\xvec}, \cdots,\node_{[i],\xvec})$, where $\node_{[i]}$ is the tree node that $\xvec$ reaches at $i$-th level of the tree 
, and where $\node_{[0]}$ is the root node.
We will usually suppress the dependence on $\xvec$ if this is clear from the context and merely write $\treepath_{(i)}(\xvec) \triangleq (\node_{[0]}, \node_{[1]}, \cdots,\node_{[i]})$.

Let $\yvec \triangleq \sigma_{\tree_{(i)}}(\xvec)$ denote the output of our TSP forward evaluation.

For all $\xvec \in \zset^d$, let $\yvec \triangleq \sigma_{\tree_{(M)}}(\xvec) \equiv \pi_{\node_{[M]}} \circ \cdots \circ \pi_{\node_{[1]}} \circ \pi_{\node_{[0]}}(\xvec)$ denote the output of our TSP forward evaluation.
We want to prove that $\treepath_{(M)}(\xvec)$ can be recovered from $\yvec$ and $\sigma_{\tree}$. 
Let $\treepath_{(i)}'(\yvec) \triangleq \node_{[0]}', \node_{[1]}', \cdots,\node_{[i]}'$ where we traverse the decision tree of $\tree$ \emph{without} applying node permutations 
If we can prove $\treepath_{(M)}'(\yvec)=\treepath_{(M)}(\xvec)$
then we are done.

\paragraph{Inductive hypothesis} For all $i \in \{0,1,\cdots,M\}$, $\treepath_{(i)}'(\yvec)=\treepath_{(i)}(\xvec)$.

\paragraph{Base case ($i=0$)} Since both $\yvec$ and $\xvec$ start at the root node, then the path up to $i=0$ is the same.

\paragraph{Induction step} We need to prove that if $\treepath_{(i)}'(\yvec) = \treepath_{(i)}(\xvec)$, then $\treepath_{(i+1)}'(\yvec) = \treepath_{(i+1)}(\xvec)$.

\paragraph{Proof} From assumption of inductive hypothesis, we know that $\node_{[i]}' = \node_{[i]}$ (i.e., $\node_{[i]}$ and the corresponding $\node_{[i]}'$ are the same node for level $i$).

Let $\xvec^{(i)} = \sigma_{\tree(i)}(\xvec) = \pi_{\node_{[i]}} \circ \cdots \circ \pi_{\node_{[1]}} \circ \pi_{\node_{[0]}}(\xvec)$ where $\xvec^{(M)} \equiv \yvec$ 

Note that only the value $v$ of the split feature $s$ for both $\xvec^{(i)}$ and $\yvec$ is relevant for determining whether to go left or right.

We prove that the chosen nodes are the same using \emph{contradiction}.

\hspace{10pt} Suppose $x_s^{(i)} = v$ such that $\xvec$ goes left and suppose $y_s \neq v$ such that $\yvec$ would go right. 

We know that the $s$-th part of the domain of the left child has only $v$ in it, i.e., $\mathcal{D}_s(\node_{\lleft}) = \{v\}$.
Also, we know that domains of children are always smaller disjoint subsets of the parent, i.e., $\mathcal{D}_s(\node_{[l]}) \subseteq \mathcal{D}_s(\node_{\lleft})$ for all $l>i$. 

Thus, $x_s^{(l)}=v$ for all $l>i$ because the invertibility constraint ensures that we cannot permute a value inside the domain to a value outside the domain.
However, this is a contradiction to our assumption that $x^{(M)}_s \equiv y_s \neq v$ .
Therefore, if $\xvec$ goes left, then $\yvec$ will also go left.

\hspace{10pt} In a similar way, now suppose $x_s^{(i)} \neq v$ such that $\xvec$ goes right and suppose $y_s = v$ such that $\yvec$ would go left. 
The $s$-th part of the domain of the right child has $\mathcal{D}_s(\node_{\rright})= \{a: a \neq v, a \in \mathcal{D}_s(\node)\}$.
Again, $\mathcal{D}_s(\node_{[l]}) \subseteq \mathcal{D}_s(\node_{\rright})$ for all $l>i$ because every child is a subset of the parent domain.

Therefore, $\xvec_s^{(l)}\in \mathcal{D}_s(\node_{\rright})$ for all $l>i$, and thus in particular $x_s^{(M)} \in \mathcal{D}_s(\node_{\rright})$, where $x_s^{(M)} \equiv y_s$ by definition.
However, this contradicts our assumption that $y_s=v$ (i.e., goes left) because $v \not\in \mathcal{D}_s(\node_{\rright})$.

Hence, if $x$ goes right, $y$ will also go right. 

\hspace{10pt} Combining these two we get that $\node_{[i+1]}' = \node_{[i+1]}$ 
(i.e., they will both go left or both go right), and thus we can recover the path for $i+1$ by adding the child node to the path for $i$, i.e., $\treepath'_{(i+1)}(\yvec) = \treepath_{(i+1)}(\xvec)$. 
This proves our inductive step and concludes the proof of the lemma.
\end{proof}

\subsection{Proof of \cref{thm:invertibility-constraint}}

\begin{proof}
\cref{lem:recoverability-of-tree-path} (proven above) states that the TSP traversal path for any input can be recovered from the output. Thus, the output path is identical to the input path $\mathcal{P}_{(i)}'(\yvec) = \mathcal{P}_{(i)}(\xvec)$.
Therefore, $\xvec = \sigma_{\tree}^{-1}(\yvec) \equiv \pi_{\node_{[0]}'}^{-1} \circ \cdots \circ \pi_{\node_{[m-1]}'}^{-1} \circ \pi_{\node_{[m]}'}^{-1}(\yvec)$ because each node permutation is itself invertible by the definition of a permutation.
This can be seen as traversing the tree from the corresponding leaf node to the root node and applying the inverse node permutations along the path.

\end{proof}

\subsection{Necessary and Sufficient Condition for Invertibility with Proof}

For better understanding of TSPs, we now present a condition (i.e., disjoint ranges of leaf nodes) that is both necessary and sufficient for invertibility.
Note that this theorem can be easily used to prove \cref{thm:invertibility-constraint} as a corollary because the original invertibility constraint set is a subset of the disjoint range of leaf nodes constraint.
However, the proof here does not provide an efficient algorithm for determining the leaf node for the inverse, while the proof of \cref{lem:recoverability-of-tree-path}  does provide an efficient algorithm (i.e., merely traverse the tree as described in the lemma proof to determine the path).

\begin{theorem}[TSP Necessary and Sufficient Invertibility Constraint]
\label{thm:necessary-and-sufficient-invertibility-constraint}
A TSP is invertible if and only if the range of each leaf node is disjoint from all other leaf nodes, i.e., $\mathcal{R}(\node) \cap \mathcal{R}(\node') = \emptyset,  \forall \node, \node' \in \mathcal{T}_{\textnormal{leaf}}$ such that $\node \neq \node'$, where $\mathcal{T}_{\textnormal{leaf}}$ are the set of leaves in the TSP tree. 
\end{theorem}

\begin{proof}

\textbf{We use a constructive proof for the if direction (sufficiency).}  
Because each of the permutations themselves are invertible, the primary challenge is \emph{showing that we can find the right path through the tree} (as there could be multiple paths if we don't consider domain related constraints) 

If the disjoint range condition is 
satisfied, then each possible output can be mapped to one of the leaves, 
i.e. $\node_{[M]}$ is the leaf node such that $\yvec \in \mathcal{R}(\node)$.
Given the leaf node, there is only one possible path through the decision tree back to the root.

Thus, the inverse can be computed by traversing from the leaf node to the root node and applying the inverse of each node's permutation.\\

\textbf{To prove the only-if direction (necessity), we will use a proof by contradiction}.

Suppose a TSP is invertible but the disjoint range condition is not satisfied, then $\mathcal{R}(\node) \cap \mathcal{R}(\node') \neq \emptyset$.
Therefore, there exists an output $\yvec$ that is in the range of two leaf nodes, i.e., $\exists \yvec$ such that $\yvec \in \mathcal{R}(\node)$ and $\yvec \in \mathcal{R}(\node')$ where $\node \neq \node'$.
Yet, each input traverses the TSP tree in a deterministic way and thus each unique input will always arrive at the same leaf node.
Therefore, there must exist two distinct inputs $\xvec \neq \xvec'$ such that $\sigma_{\tree}(\xvec) = \sigma_{\tree}(\xvec')=\yvec$.
This means that two distinct inputs map to the same output (i.e., not one-to-one mapping) and violates invertibility.
However, this contradicts our assumption that the TSP is invertible.
\end{proof}

\section{Proof of optimality of rank consistency}
\label{rank_proofs}

\subsection{Proof of \cref{thm:optimality-of-rank-consistent-tsps}}
\OptimalityOfRankConsistentTsps*
\begin{proof} 
The proof is by \emph{contradiction}. 

We will prove that we can construct another equivalent TSP from this TSP that will yield a better log likelihood which will lead to a contradiction of optimality. In particular, there exists a pair of counts that can be switched to be consistent with the global rank that will yield a better TSP.

Suppose a TSP is optimal among TSPs that are tree equivalent but rank consistency was not satisfied.  
Without loss of generality, we will assume the global rank permutations are the identity (i.e., $\pi_j = \pi_{\text{Id}}, \forall j$)  so we can simplify notation.
This would mean that there exists a rank inconsistent (RI) node, i.e., $\exists \tilde{\node}, j \in \{0,\cdots,\ndim-1\}, (a,b) \in \{ (a,b): a < b, a \in \domain(\tilde{\node}), b \in \domain(\tilde{\node}) \}$ such that 
\begin{align}
    c_{\tilde{\node}}(j,a) > c_{\tilde{\node}}(j,b) \,.
\end{align}
Let $\tilde{\pi}$ be the permutation that switches $a$ and $b$ of the $j$-th dimension from the above but leaves all other discrete values untouched.
Now, we alter the current node's permutation and split values as such:
\begin{align}
    \pi_{\tilde{\node}}^{\new} &\triangleq \tilde{\pi} \circ \pi_{\tilde{\node}} \\
    v_{\tilde{\node}}^{\new} &\triangleq \{ \tilde{\pi}(v) : v \in v_{\tilde{\node}} \}
\end{align}
and we alter the permutations and split values at descendant nodes as follows, i.e., $\forall \node \in \desc(\tilde{\node})$ 
\begin{align}
    \pi_{\node}^{\new} &\triangleq \tilde{\pi} \circ \pi_{\node} \circ \tilde{\pi}^{-1}  \\
    v_{\node}^{\new} &\triangleq \{ \tilde{\pi}(v) : v \in v_{\node} \} \,.
\end{align}
First, we show that this modified TSP is tree equivalent by \emph{induction} on the depth of the descendant subtree after the $\tilde{\node}$ node, which is at depth $m$. Note that any nodes that are not $\tilde{\node}$ or its descendants are unchanged and thus already satisfy the tree equivalence property.

The \textbf{base case} is then simply to determine if a configuration that would go left in the original TSP will go left in the new tree for the current $\tilde{\node}$.

Let $\xvec \in \domain(\tilde{\node})$ go left in the original tree, i.e., $\pi_{\tilde{\node},s}(\xvec) \in v_{\tilde{\node}}$.
From this we have that:
\begin{align}
    \pi_{\tilde{\node},s}(\xvec) &\in v_{\tilde{\node}} \\
    \Leftrightarrow \tilde{\pi} \circ \pi_{\tilde{\node},s}(\xvec) &\in v_{\tilde{\node}}^{\new} \\
    \Leftrightarrow \pi_{\tilde{\node},s}^{\new}(\xvec) &\in v_{\tilde{\node}}^{\new} \,,
\end{align}

where the second line is by the definition of $v_{\tilde{\node}}^{\new}$, and the third is by the definition of $\pi_{\tilde{\node}}^{\new}\equiv\tilde{\pi} \circ \pi_{\tilde{\node}}$.
We also note that for these leaf nodes, the domain and range are the same.
Thus, we have that $\forall \xvec, \sigma_{\tree(m+1)}(\xvec) \in \domain(\tilde{\node}_{\lleft}) \Leftrightarrow \sigma_{\tree(m+1)}^{\new}(\xvec) \in \domain(\tilde{\node}^{\new}_{\lleft})$ where the size of the new domain is equivalent $|\domain(\node_{\lleft})|=|\domain(\node^{\new}_{\lleft})|$ (and similarly for the right node), i.e., any configuration that went left in the original TSP will go left in the new TSP where the depth of the subtree of $\tilde{\node}$ is only 1 (i.e., a single split).

For the \textbf{induction} step, suppose the induction hypothesis holds for up to  depth $m+k$, we will prove that it holds for depth $m+k+1$.
Let $\node$ be a node at depth $k$ away from $\tilde{\node}$ which itself is at depth $m$, and let $\sigma_{\tree(m)}$ be the TSP permutation up to depth $m$ 
and $\sigma_{\tree(m+k)}$ is the TSP permutation up to depth $m+k$. 

By the inductive hypothesis, we have that the domains of the nodes are equivalent:
\begin{align}
    \sigma_{\tree(m+k)}(\xvec) &\in \domain(\node) \\
    \Leftrightarrow \sigma_{\tree(m+k)}^{\new}(\xvec) &\in \domain^{\new}(\node) \,.
\end{align}

First, let's show that $\sigma_{\tree(m+k)}^{\new} \equiv \tilde{\pi} \circ \sigma_{\tree(m+k)}$ as follows.
We will denote the tree traversal path of an input $\xvec$ to be $\treepath(\xvec) \triangleq (\node_{[0]}, \node_{[1]}, \cdots,\node_{[M]})$, where $\node_{[M]}$ is the leaf node at max depth $M$ that the input reaches and $\node_{[0]}$ is the root node. Also, note that 
$\tilde{\node} \equiv \node_{{[m]}}$.

We define $\pi_{\anc(\node_{[m]})}(x)\triangleq \sigma_{\tree(m-1)}(\xvec)$  (i.e. $\pi_{\anc(\node_{[m]})}(x)$ is the evaluation of permutations of all ancestral nodes of $\node_{[m]}$) to add intuition to the proofs.

\begin{align}
    \sigma_{\tree(m+k)}(\xvec)&\equiv \pi_{\node_{[m+k]}} \circ \cdots \circ \pi_{\node_{[1]}} \circ \pi_{\node_{[0]}}(\xvec) \\
    &\equiv \pi_{\node_{[m+k]}} \circ \cdots \circ \pi_{\node_{[m]}} \circ \pi_{\anc(\tilde{\node})}(\xvec) \\
    \sigma_{\tree(m+k)}^{\new}(\xvec)&\equiv \pi_{\node_{[m+k]}}^{\new} \circ \cdots \circ \pi_{\node_{[m]}}^{\new} \circ \pi_{\anc(\tilde{\node})}(\xvec) \,.
\end{align}
Now let's expand the new part:
\begin{align}
    \sigma_{\tree(m+k)}^{\new}(\xvec) &=\pi_{\node_{[m+k]}}^{\new} \circ \cdots \circ \pi_{\node_{[m+1]}}^{\new} \circ \pi_{\node_{[m]}}^{\new} \circ \pi_{\anc(\tilde{\node})}(\xvec) \\
    &=(\tilde{\pi} \circ \pi_{\node_{[m+k]}}\circ \tilde{\pi}^{-1}) \circ \cdots \circ (\tilde{\pi} \circ \pi_{\node_{[m+1]}} \circ \tilde{\pi}^{-1}) \notag\\
    &\quad\quad \circ (\tilde{\pi} \circ \pi_{\node_{[m]}}) \circ \pi_{\anc(\tilde{\node})}(\xvec) \\
    &=\tilde{\pi} \circ \pi_{\node_{[m+k]}} \circ \cdots \circ \pi_{\node_{[m+1]}} \notag\\
    &\quad\quad \circ \pi_{\node_{[m]}} \circ \pi_{\anc(\tilde{\node})}(\xvec)\\
    &= \tilde{\pi} \circ \sigma_{\tree(m+k)} \,,
\end{align}

where the second line is by the definition of the new node permutations, the third line is by noticing that this composition is a telescoping composition where the inner $\tilde{\pi}^{-1} \circ \tilde{\pi}$ terms cancel out, and the last line is by the definition of $\sigma_{\tree(m+k)}$.

Next, we show that given our definitions of the descendant nodes above, the same configuration will go left in the new tree, 
\begin{align}
    \sigma_{\tree(m+k)}(\xvec_s) &\in v \\
    \tilde{\pi} \circ \sigma_{\tree(m+k)}(\xvec_s) &\in v^{\new} \\
    \sigma_{\tree(m+k)}^{\new}(\xvec_s) &\in v^{\new} \,,
\end{align}
where the first line is by assumption, the second line is by the definition of $v^{\new}$ and the third line is by our derivation above about the relationship between the new and old permutations.
Thus, again, the same inputs will go left and right and thus the inductive hypothesis holds for $m+1$, i.e., $\forall \node$ at a depth of $m+k+1$:
\begin{align}
    \sigma_{\tree(m+k+1)}(\xvec) &\in \domain(\node) 
    \Leftrightarrow \sigma_{\tree(m+k+1)}^{\new}(\xvec) \in \domain^{\new}(\node) \,.
\end{align}

For the next part of the proof, we need to prove that this equivalent TSP has better negative log likelihood.
Similar to the above derivation we can know that $\sigma_{\tree}^{\new}(\xvec) = \tilde{\pi} \circ \sigma_{\tree}(\xvec), \forall \xvec \in \domain(\tilde{\node})$.
And because $\tilde{\pi}$ only swaps two configurations, then the node counts are all equivalent except that $c_{\node}^{\new}(s,a) = c_{\node}(s,b)$ and  $c_{\node}^{\new}(s,b) = c_{\node}(s,a)$ for all nodes that are equal to $\tilde{\node}$ or descendants of $\tilde{\node}$.
Now we also note that the maximum likelihood solution is actually equivalent to minimizing the sum of feature-wise empirical entropies:
\begin{align}
    &\min_{Q} \frac{1}{n}\sum_{i=1}^n -\log Q(\zvec_i) \\
    &=\min_{Q} \frac{1}{n}\sum_{i=1}^n \sum_{j=1}^k -\log Q_j(z_{i,j})  \\
    &=\sum_{j=1}^k \min_{Q_j} \frac{1}{n}\sum_{i=1}^n  -\log Q_j(z_{i,j})  \\
    &=\sum_{j=1}^k \min_{Q_j} \E_{\hat{Q}_j}[-\log Q_j(z_{i,j})]  \label{eqn:min-cross-entropy} \\
    &=\sum_{j=1}^k \E_{\hat{Q}_j}[-\log \hat{Q}_j(z_{i,j})]  \\
    &=\sum_{j=1}^k H(\hat{Q}_j) \,,
\end{align}

where $\zvec = \sigma_{\tree}(\xvec)$ and $\hat{Q}_j$ is the empirical distribution of the $j$-th dimension of $\zvec$ (i.e., merely the empirical probabilities based on normalizing the counts).
The first equals is by the assumption of independence, the second is by noting that the minimization is decomposable, the third is by the definition of an empirical expectation, the fourth is by the optimal solution (i.e., the empirical probabilities is the best MLE estimate given the empirical distribution) and the last is by the definition of entropy.
Note that \autoref{eqn:min-cross-entropy} can actually be seen as another definition of entropy.
Thus, we only need to compare the entropies for feature $s$ (since all others are equal) $H(\hat{Q}_s^{\new}) < H(\hat{Q}_s)$.
Without loss of generality, let us suppose that $k=2$ (i.e., $\hat{Q}_s$ is a Bernoulli distribution).
Now let $p_Q \triangleq \hat{Q}_s(a)$ and $p_Q^{\new} \triangleq \hat{Q}^{\new}_s(a)$, dropping the notation on $s$ because this is the only variable the changes.
Because the empirical counts for $a$ are larger in the new distribution, we know that $p_Q^{\new} > p_Q$ and additionally we know that $p_Q > (1-p_Q)$ because $c_{\tilde{\node}}(s,a) > c_{\tilde{\node}}(s,b)$ by assumption.
Given that $p_Q > (1-p_Q)$, the derivative of entropy is negative, i.e., $\frac{dH(p_Q)}{dp} = -\log p_Q + \log (1-p_Q) < 0$ because $\log$ is a monotonically increasing function (i.e., $\log p_Q > \log(1-p_Q)$).
We can form a linear upper bound on $H$ by the first order Taylor series expansion around $p_0$:
\begin{align}
    H(p) &\leq H(p_0) + \frac{dH(p_0)}{dp}(p-p_0)
\end{align}
and derive that the entropy of the new is lower:
\begin{align}
    H(p_Q^{\new}) 
    &\leq H(p_Q) + \frac{dH(p_Q)}{dp} (p_Q^{\new}-p_Q)
    < H(p_Q) \,,
\end{align}
where the first line is by the concavity of $H(p)$ and the second is by the fact that $p_Q^{\new}-p_Q > 0$ while $\frac{dH(p_Q)}{dp}<0$.
Thus, the newly constructed TSP is tree equivalent yet it has a lower negative log likelihood (or equivalently lower feature-wise entropy).
Yet, this is a contradiction to our assumption that the original TSP was optimal.
\end{proof}

\section{Tree Equivalence }
\label{tree_eq}
\TreeEquivalence*

\begin{proof}[Proof that $\equivtree$ is indeed an equivalence relation]

\hfill \break

\textbf{Reflexive property} ($\forall \tree, \tree \equivtree \tree$)
This property is trivial by inspection of the definition. For any $\tree$, clearly any tree has the same graph structure as itself and $\forall j, \xvec, \, \sigma_{\tree}(\xvec) \in \domain(\node_j) \Leftrightarrow \sigma_{\tree}(\xvec) \in \domain(\node_j)$, where both statements are almost trivial.

\textbf{Symmetric property} ($\forall \tree_A, \tree_B, \tree_A \equivtree \tree_B \Leftarrow \tree_B \equivtree \tree_A$)
Again, this is easy to prove.  Suppose $\tree_A \equivtree \tree_B$, this means that $\tree_A$ and $\tree_B$ have the same graph structure and that $\forall j, \xvec, \, \sigma_{\tree_A}(\xvec) \in \domain(\node_j^{(A)}) \Leftrightarrow \sigma_{\tree_B}(\xvec) \in \domain(\node_j^{(B)})$. Because the structure is the same, then $\tree_B$ has the same structure as $\tree_A$ and we know that $\forall j, \xvec, \, \sigma_{\tree_B}(\xvec) \in \domain(\node_j^{(B)}) \Leftrightarrow \sigma_{\tree_A}(\xvec) \in \domain(\node_j^{(A)})$. Thus, $\tree_B \equivtree \tree_A$.

\textbf{Transitive property} ($\forall \tree_A, \tree_B, \tree_C, \text{ if } \tree_A \equivtree \tree_B \text{ and } \tree_B \equivtree \tree_C, \text{ then } \tree_A \equivtree \tree_C$)
First, we check the transitivity of property (1): If $\tree_A$ has the same graph structure as $\tree_B$ and $\tree_B$ has the same graph structure as $\tree_C$, then $\tree_A$ and $\tree_C$ have the same graph structure. 
Second, we check the transitivity of property (2):
From $\tree_A \equivtree \tree_B$, we know that $\forall j, \xvec, \, \sigma_{\tree_A}(\xvec) \in \domain(\node_j^{(A)}) \Leftrightarrow \sigma_{\tree_B}(\xvec) \in \domain(\node_j^{(B)})$ and from $\tree_B \equivtree \tree_C$, we know that $\forall j, \xvec, \, \sigma_{\tree_B}(\xvec) \in \domain(\node_j^{(B)}) \Leftrightarrow \sigma_{\tree_C}(\xvec) \in \domain(\node_j^{(C)})$. Because each of these statements are if and only if (i.e., $\Leftrightarrow$), we can clearly conclude that 
$\forall j, \xvec, \, \sigma_{\tree_A}(\xvec) \in \domain(\node_j^{(A)}) \Leftrightarrow \sigma_{\tree_C}(\xvec) \in \domain(\node_j^{(C)})$, which proves the second property is transitive.  Thus, we can conclude that $\tree_A \equivtree \tree_C$.
 
\end{proof}

This equivalence relation sheds additional interpretation of our algorithm.  The splitting criteria determines one tree in the equivalence class of TSPs defined by the tree equivalence relation (specifically one in the equivalence class where all the permutations are the identity).
Our other algorithms then finds the optimal \emph{equivalent} tree in the equivalence class that maintains this split structure, but changes the permutations to be the optimal with respect to NLL.

Informally, this equivalence relation means that any configuration will reach the same node for any tree in the equivalence class.
However, the permutations and split criteria might be different for each TSP in the equivalence class---and in particular, the choice of which permutations and split criteria will affect the NLL of the model.

We present an example of two tree equivalent trees in Figures ~\ref{tree_eq}

\begin{figure}
    \centering
    \subfloat[\centering Tree A ]{{\includegraphics[width=7cm]{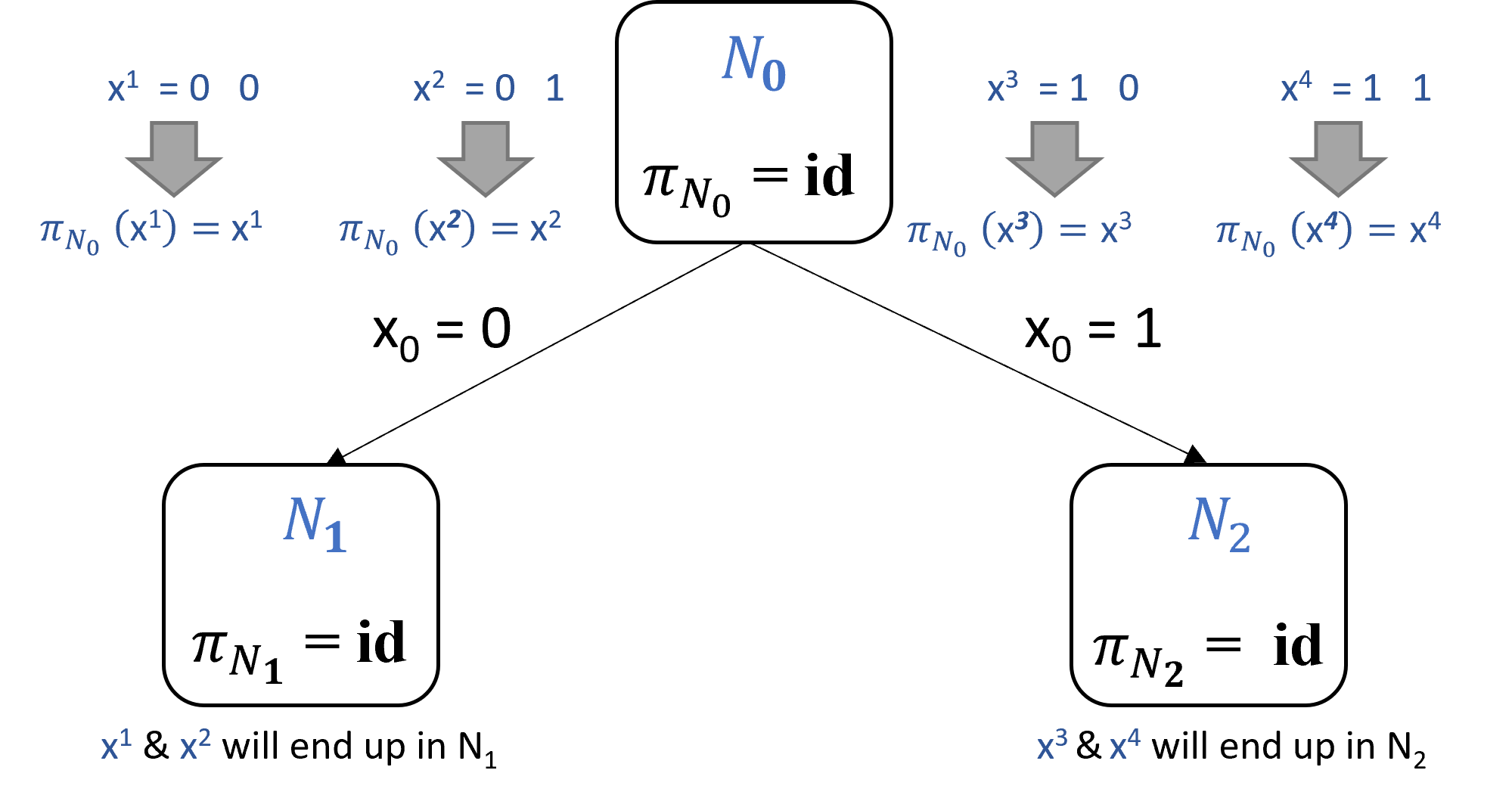} }}
    \qquad
    \subfloat[\centering Tree B]{{\includegraphics[width=8.5cm]{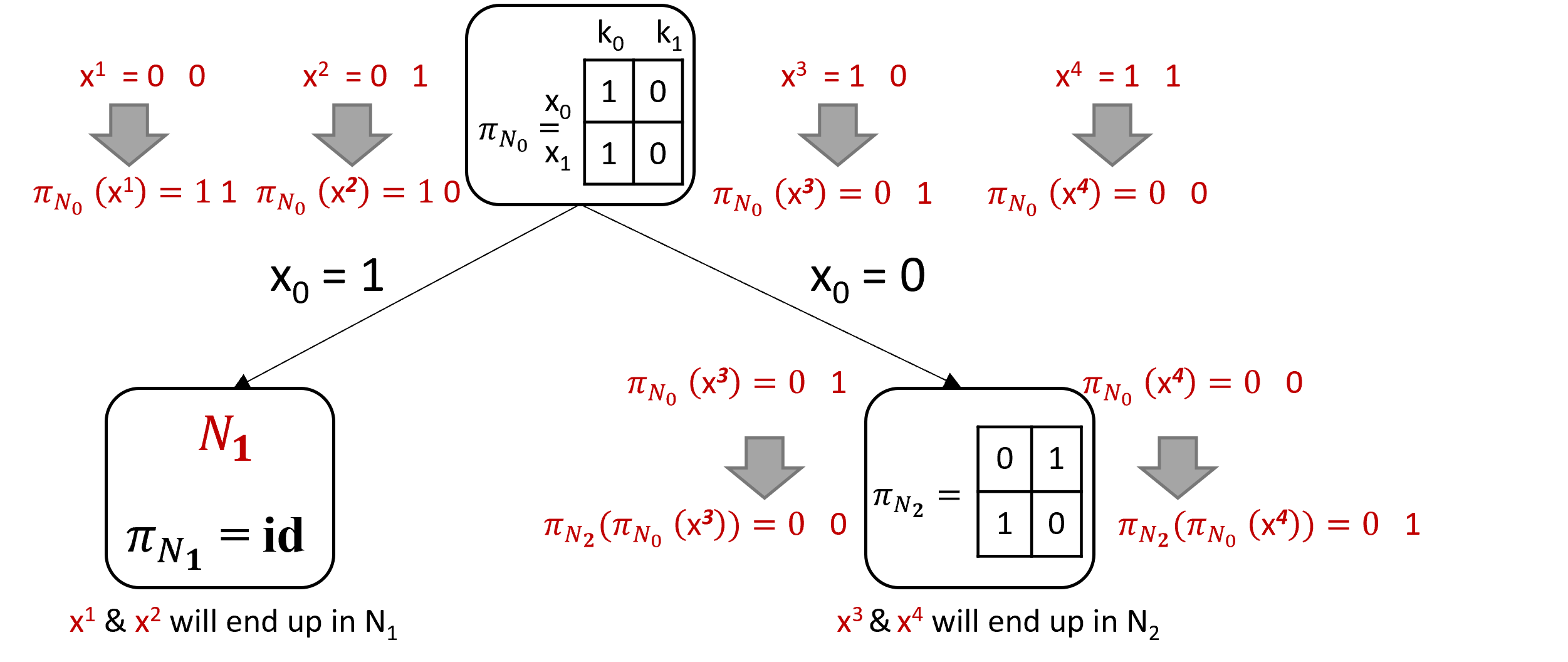} }}
    \qquad
    \caption{ An example of two equivalent trees a) Tree A b) Tree B}
    \label{fig:tree_eq}
\end{figure}

\section{Algorithms pseudo code}
\label{code}
We give the additional pseudo-code for the complete algorithm here.
At a high level, \autoref{alg:learn-tsp}:LearnTSP is the complete algorithm and calls all the other algorithms, it creates a root node, then calls \autoref{alg:construct-tree}:ConstructTree which is for constructing the tree, then learn the local permutations (\autoref{alg:learn-local-permutations}: LearnLocalPermutations) discussed in section~\ref{twopass}), then find the equivalent tree structure (\autoref{alg:construct-equivalent-tree}: ConstructEquivalentTree discussed in section~\ref{twopass}). Finally returning the node of the equivalent tree, which is the node to the final "correct" TSP. \autoref{alg:learn-tsp}:LearnTSP builds the tree recursively given the specified maximum depth of the tree. It utilizes \autoref{alg:find-best-split}:FindBestSplit to find the best split and it also outputs the counts of the categorical values at each node (which will be used by the later algorithms). \autoref{alg:find-best-split}:FindBestSplit simply iterates over all possible splits (that is all dimensions and category values) calculates the scores for each split according to the scoring function (which can be  GLP, or RND. For RND it will just randomly select a dimension to split on and a split value without the need to calculate any scores), following that it shows the pair of split feature and split value that resulted in the best score, and splits the data accordingly and also calculates the domains of the children that are created by this split.

\section{Proofs for Algorithms}
\label{alg_proofs}
Before we prove \autoref{thm:algorithm-optimal}, we first prove several important lemmas.

\subsection{Proof that \autoref{alg:construct-equivalent-tree} constructs an equivalent tree}
\ConstructEquivalentTree*

\begin{proof}
We will prove by \textbf{induction} on the tree depth.  
\paragraph{Base case ($\depth=1$)}
For the base case of depth equal to one, we only have three nodes: the parent node $\node_0$ and two children nodes, where each has an identity permutation $\pi_{\node_i} \triangleq \pi_{\text{Id}}, \forall i$ (by construction from Algorithm 1) 
but also a local permutation $\modpi_{\node_i}$ (by Algorithm 4) . We will distinguish the new nodes created by ConstructEquivalentTree with superscript new for example, the root node will be $\node_0^{\new}$. Each of these nodes has a $\pi \triangleq \pi_{\anc} \circ \modpi \circ \pi_{\anc}^{-1}$ 
(where $\pi_{\anc(\node_0)} \triangleq \pi_{\text{Id}}$), and $v_{\node}^{\new} = \{\modpi \circ \pi_{\anc(\node)}(v) :  v \in v_{\node} \}$ 
It is sufficient to verify that the same input configurations that went left in the original tree will go to the left in the equivalent tree.

First, note that for $\depth=1$, $\pi_{\anc(\node_0)} = \pi_{\text{Id}}$ and thus $\pi_{\node_0}^{\new} = \modpi_{\node_{[0]}}$.
Let $\xvec$ be a configuration that went left in the original TSP.
This means that 
\begin{align}
  x_j &\in v_{\node} \\
  \Leftrightarrow \modpi_{\node_0}(x_j) &\in v^{\new}_{\node} \,,
\end{align}
where the first implications is by the construction of $v_{\node}^{\new}$.
To ensure invertibility, we must also verify that the domain and range are equal for each new node.
The root node has a domain of the entire space and thus any permutation will ensure that the domain and range are equal.
Thus, we only need to verify that the new child permutations only permute their new domains, i.e., $\xvec \in \domain(\node^{\new}) \Leftrightarrow \pi^{\new}(\xvec) \in \domain(\node^{\new})$. 
First, we note that $\range(\modpi) \equiv \domain$ by construction of $\modpi$, i.e., $\xvec \in \domain \Leftrightarrow \modpi(\xvec) \in \domain$.
We also note that the same input $\xvec$ reach this node (where all the permutations were the identity), i.e., $\forall \xvec, \sigma_{\tree}(\xvec) \equiv \xvec \in \domain(\node) \Leftrightarrow \sigma_{\tree}^{\new}(\xvec) \in \domain(\node^{\new})$, but now there are permutations applied at the ancestors of the node $\node$, we have the following:
\begin{align}
    \domain(\node^{\new}) = \{\pi_{\anc}(\xvec) : \xvec \in \domain(\node)\} \,.
    \label{eqn:new-domain-anc}
\end{align}
We prove for the left child as the right child can be proved in the same way.
\begin{align}
    \xvec &\in \domain(\node^{\new}) \\
    \Leftrightarrow \pi_{\anc}^{-1}(\xvec) &\in \domain(\node) \\ 
    \Leftrightarrow \modpi \circ \pi_{\anc}^{-1}(\xvec) &\in \domain(\node) \\
    \Leftrightarrow \pi_{\anc} \circ \modpi \circ \pi_{\anc}^{-1}(\xvec^{\new}) &\in \domain(\node^{\new}) \,
\end{align}

where the first line is from the properties of one-to-one mappings and \autoref{eqn:new-domain-anc}, the second line is by the property of $\modpi$ such that the domain and range are the same , and the third line is also by \autoref{eqn:new-domain-anc}.
Thus, the new TSP has equivalent tree structure and is a valid TSP.

\paragraph{Inductive step}
The key step is to prove that if the property is true all nodes of till depth $m$, any nodes at depth $m+1$ also have equivalent tree structure.
Without loss of generality, we choose the split of one node at depth $m$ and show that the same points would go left (all other nodes at depth $m$ can be proved similarly).
Suppose that $\xvec \in \domain(\node)$ went to the left child in our original TSP, i.e., $x_s \in v$.
By our inductive hypothesis $\pi_{\anc}(\xvec) \in \domain(\node^{\new})$.

Let $\xvec^{\new} \equiv \modpi \circ \pi_{\anc}(\xvec_s)$ be the equivalent configuration in the new TSP, we can derive that:

\begin{align}
    x_s &\in v_\node \\
    \Leftrightarrow \modpi \circ \pi_{\anc}(\xvec_s) &\in v_\node^{\new}  \\
    \Leftrightarrow x_s^{\new} &\in v_\node^{\new}  \,, 
\end{align}
where the second line is by the definition of $v^{\new}$ and the third line is by the definition of $\xvec^{\new}$.
Thus, the samples that went left in the original TSP will also go left in this new TSP.
For this step, the domains of the node can be validated in the same way as in the base case since there is nothing special in the inductive case.

\end{proof}

\subsection{Proof that new TSP is equal to applying local permutations in reverse order}
We will need the following simple lemma about the structure of the new TSP permutation to prove our result in \autoref{thm:learn-local-permutations}.
\begin{restatable}{lemma}{LocalReverseOrder}
\label{Lemma:local-reverse-order}
The new TSP constructed from our algorithms is equal to applying the local permutations in reverse order:
\begin{align}
    \forall M > 0,\quad \sigma_{\tree_{(M)}}^{\new}
    &= \pi_{\node_{[M]}}^{\new} \circ \pi_{\node_{[M-1]}}^{\new} \circ\cdots\circ \pi_{\node_{[0]}}^{\new} \nonumber
    = \tilde{\pi}_{\node_{[0]}} \circ \tilde{\pi}_{\node_{[1]}}\circ\cdots\circ \tilde{\pi}_{\node_{[M]}} \,.
\end{align}
Also, one specific case of this is for ancestor permutations:
\begin{align}
    \pi_{\anc(\node_{[M]})}^\new = \tilde{\pi}_{\node_{[0]}} \circ \tilde{\pi}_{\node_{[1]}}\circ\cdots\circ \tilde{\pi}_{\node_{[M-1]}} \,.
\end{align}
\end{restatable}

\begin{proof}
    We will prove by induction. We will first recall some definitions:
    \begin{align}
        \pi &\triangleq \pi_{\anc} \circ \modpi \circ \pi_{\anc}^{-1} \\
        \pi^{-1} &\triangleq \pi_{\anc} \circ \modpi^{-1} \circ \pi_{\anc}^{-1} \\
        \pi_{\anc(\node_{[m]})} &\triangleq \pi_{\node_{[m-1]}} \circ \pi_{\node_{[m-2]}} \circ \cdots \circ \pi_{\node_{[1]}} \circ \pi_{\node_{[0]}}
        \label{eqn:anc_def}\\
        \pi_{\anc(\node_{[m]})}^{-1} &\triangleq \pi_{\node_{[0]}}^{-1} \circ \pi_{\node_{[1]}}^{-1} \circ \cdots \circ \pi_{\node_{[m-2]}}^{-1} \circ \pi_{\node_{[m-1]}}^{-1}
        \label{eqn:inverse_anc_def}
    \end{align}
    
    \paragraph{Base case ($TSP\; \depth=1$)}
    
    We define the new TSP for depth of 1 
    \begin{align}
        \sigma_{\tree_{(1)}}^{\new}
        &= \pi_{\node_{[1]}}^{\new} \circ \pi_{\node_{[0]}}^{\new}\\
        &= (\pi_{\anc(\node_{[1]})}^{\new} \circ \modpi_{\node_{[0]}} \circ  {\pi_{\anc(\node_{[1]})}^{\new}}^{-1}) \circ
        \pi_{\node_{[0]}}^{\new}\\
        &= \pi_{\node_{[0]}}^{\new} \circ \modpi_{\node_{[1]}} \circ  {\pi_{\node_{[0]}}^{\new}}^{-1} \circ
        \pi_{\node_{[0]}}^{\new}\\
        &= \pi_{\node_{[0]}}^{\new} \circ \modpi_{\node_{[1]}}\\
        &=(\pi_{\anc(\node_{[0]})}^{\new} \circ \modpi_{\node_{[0]}} \circ  {\pi_{\anc(\node_{[0]})}^{\new}}^{-1} ) \circ
        \modpi_{\node_{[1]}}\\
        &=\modpi_{\node_{[0]}}\circ\modpi_{\node_{[1]}}\,,
    \end{align}

    where the second equality is just the definition of our permutation for $\pi_{\node_{[1]}}^{\new}$, the third equality is the expansion of the ancestor permutation definitions, the fourth equality is because ${\pi_{\node_{[0]}}^{\new}}^{-1} \circ
    \pi_{\node_{[0]}}^{\new}$ cancel each other. The equation is simplified to the last equation since there exists no ancestor nodes for $\node_{[0]}$, thus we can assign ${\pi_{\anc(\node_{[0]})}^{\new}}$ as an identity. Therefore, Lemma \ref{Lemma:local-reverse-order} holds for the base case.
    
    Furthermore, we can define the new TSP at depth 1 as the new ancestor permutations for $\node_{[2]}$:
    
    \begin{align}
        \pi_{\anc(\node_{[2]})}^{\new} &= \sigma_{\tree_{(1)}}^{\new}= \modpi_{\node_{[0]}}\circ\modpi_{\node_{[1]}}\,,
    \end{align}

    and we can give a more general definition:
    \begin{align}
        \pi_{\anc(\node_{[m]})}^{\new} &= \sigma_{\tree_{(m-1)}}^{\new}= \modpi_{\node_{[0]}}\circ\modpi_{\node_{[1]}}\circ\cdots\modpi_{\node_{[m-1]}}\,,
            \label{eqn:general_new_ancestor_perm}
    \end{align}
    
    \paragraph{Induction step ($TSP\; \depth=m \Rightarrow TSP\;\depth=m+1$)} 
    
    Let us assume for depth of $m$, 
    Lemma \ref{Lemma:local-reverse-order} holds true. Then the new TSP for depth of $m$ using definition from \autoref{eqn:general_new_ancestor_perm} can be written as below:
    
    \begin{align}
        \sigma_{\tree_{(m)}}^{\new}&= \pi_{\anc(\node_{[m+1]})}^{\new} \label{eqn:k_depth_TSP}\\
        &= \pi_{\node_{[m]}}^{\new} \circ \pi_{\node_{[m-1]}}^{\new} \circ\cdots\circ \pi_{\node_{[0]}}^{\new} \\
        &= \modpi_{\node_{[0]}}\circ\modpi_{\node_{[1]}}\circ\cdots\modpi_{\node_{[m]}}\,.
            \label{eqn:k_depth_TSP_reversed}
    \end{align}
    
    Let us now observe our new TSP permutation for depth of $m+1$:
    \begin{align}
        \sigma_{\tree{(m+1)}}^{\new}&=\pi_{\node_{[m+1]}}^{\new} \circ \pi_{\node_{[m]}}^{\new} \circ\cdots\circ \pi_{\node_{[0]}}^{\new}\,.
    \end{align}
    Referring back to \autoref{eqn:k_depth_TSP}, we can rewrite the equality as:
    
    \begin{align}
        \sigma_{\tree_{(m+1)}}^{\new}&=\pi_{\node_{[m+1]}}^{\new} \circ \pi_{\anc(\node_{[m+1]})}^{\new}\\
        &= \pi_{\anc(\node_{[m+1]})}^{\new} \circ \modpi_{\node_{[m+1]}} \circ  {\pi_{\anc(\node_{[m+1]})}^{\new}}^{-1}\circ \pi_{\anc(\node_{[m+1]})}^{\new}\\
        &=\pi_{\anc(\node_{[m+1]})}^{\new} \circ \modpi_{\node_{[m+1]}}\\
        &= \modpi_{\node_{[0]}}\circ\modpi_{\node_{[1]}}\circ\cdots\modpi_{\node_{[m]}}\circ\modpi_{\node_{[m+1]}}\,,
    \end{align}
    
    where the third equality is given by the cancellation of the ancestor and the inverse ancestor permutation on the previous line. The final equality is the expansion of $\pi_{\anc(\node_{[m+1]})}^{\new}$ on \autoref{eqn:k_depth_TSP_reversed}. Therefore, we prove that
    \begin{align}
        \sigma_{\tree_{(M)}}^{\new}
        &= \tilde{\pi}_{\node_{[0]}} \circ \tilde{\pi}_{\node_{[1]}}\circ\cdots\circ \tilde{\pi}_{\node_{[M]}}.
        \label{eqn:permutation_equiv}
    \end{align}
    
\end{proof}

\subsection{Proof that new counts are a permutation of the local counts} 
\begin{lemma}[Relation between local and new node counts]
\label{thm:relation-local-and-new-counts}
The new node counts $c_{\node^{\new}}$ are equal to the local counts $\tilde{c}_\node$ after permutation by the ancestor permutation $\pi_{\anc(\node)}^{\new}$, i.e., $c_{\node^{\new}} = \pi_{\anc(\node)}^{\new}[ \tilde{c}_\node]$.
\end{lemma}
\begin{proof}
Consider any \textbf{leaf node} $N \equiv N_{[M]}$ at depth $M$, where we suppress the dependence on $M$ for now.
At the leaf node, the local counts are originally equal to the raw unpermuted counts based on Line 8 in \autoref{alg:construct-tree}, i.e.,
\begin{align}
    \forall j, a, \quad \tilde{c}_\node^{\textnormal{init}}(j,a) 
    = \sum_{i=1}^\nobs \one(\sigma_{\tree}(\xvec_{i})_j = a \land \sigma_{\tree}(\xvec_i) \in \domain(\node)) 
    = \sum_{i=1}^\nobs \one(x_{i,j} = a \land \xvec_i \in \domain(\node)) 
    \label{eqn:init-local-counts}
\end{align}
where the first equality is by the fact that only inputs that land in the node are counted and the second equality is because the original has all identity node permutations so $\sigma_{\tree}\equiv \pi_{\Id}$.
We now note that an ancestor permutation is a composition of independent permutations so each ancestor permutation can be split into an independent permutation over each feature, i.e.,
\begin{align}
    \pi_{\anc(\node)}(\xvec)_j = \pi_{\anc(\node),j}(x_j).
    \label{eqn:j-out-to-in}
\end{align}
We also note that:
\begin{align}
\begin{aligned}
    \sigma_{\tree}^\new 
    = \tilde{\pi}_{\node_{[0]}} \circ \tilde{\pi}_{\node_{[1]}} \circ \cdots \circ \tilde{\pi}_{\node_{[M-1]}} \circ \tilde{\pi}_{\node_{[M]}}
    = \pi_{\anc(\node_{[M]})}^\new \circ \tilde{\pi}_{\node_{[M]}}
    =\pi_{\anc(\node)}^\new \circ \tilde{\pi}_{\node} \,,
\end{aligned} \label{eqn:sigma-anc-node}
\end{align}
where the first and third equality is due to \autoref{Lemma:local-reverse-order}, and the last is by suppressing notational dependence on $[M]$.
Combining the facts above, we can derive the result for leaf nodes:
\begin{align}
    \forall j, a, \quad 
    c_{\node^{\new}}(j,a)
    &= \sum_{i=1}^\nobs \one(\sigma_{\tree}^\new(\xvec_{i})_j = a \land \sigma_{\tree}^{\new}(\xvec_i) \in \domain(\node^\new)) \\
    &= \sum_{i=1}^\nobs \one(\sigma_{\tree}^\new(\xvec_{i})_j = a \land \xvec_i \in \domain(\node)) \\
    &= \sum_{i=1}^\nobs \one(\pi_{\anc(\node)}^{\new}\circ\tilde{\pi}_{\node}(\xvec_{i})_j = a \land \xvec_i \in \domain(\node)) \\
    &= \sum_{i=1}^\nobs \one(\pi_{\anc(\node),j}^{\new}\circ\tilde{\pi}_{\node,j}(x_{i,j}) = a \land \xvec_i \in \domain(\node)) \\
    &= \sum_{i=1}^\nobs \one(x_{i,j} = \tilde{\pi}_{\node,j}^{-1} \circ \pi_{\anc(\node),j}^{\new-1}(a) \land \xvec_i \in \domain(\node)) \\
    &= \pi_{\anc(\node),j}^{\new}[ \tilde{\pi}_{\node,j}[\tilde{c}_\node^{\textnormal{init}}(j, a)]]\\
    &=\pi_{\anc(\node),j}^{\new}[ \tilde{c}_\node(j, a)] 
\end{align}
where the first equality is by definition, the second is by TSP tree equivalence, the third is by \eqref{eqn:sigma-anc-node}, the fourth is by \eqref{eqn:j-out-to-in}, the fifth is by invertibility, the sixth is by the definition of the $\pi[\cdot]$ operator and \eqref{eqn:init-local-counts}, and the last is by the PermuteCounts line in \autoref{alg:learn-local-permutations}.

The above derivation proves the statement for leaf nodes.
We now want to prove by induction that the statement holds for \textbf{non-leaf nodes}.  Without loss of generality, we will prove for one path between root and leaf by induction on a depth where $N_{[M]}$ is the leaf node and $N_{[0]}$ is the root node. Specifically, we will prove that this holds for $N_{[M-k]}$ where $k \in \{0,1,\dots, M\}$.
The base case of $k=0$ is proved above as $N_{[M]}$ is a leaf node so we only need to prove the inductive case.
Let $N\equiv N_{[M-(k+1)]}$ and let $N_\lleft$ and $N_{\rright}$ be it's left and right nodes which are both at depth $M-k$ (i.e., where the inductive hypothesis is assumed to hold).
First, we establish that the counts of a parent node are merely equal to the sum of the counts of children nodes, i.e.,
\begin{align}
\begin{aligned}
    \forall j, a, \quad c_{\node}(j,a)
    &= \sum_{i=1}^\nobs \one(\sigma_{\tree}(\xvec_i)_j = a \land \sigma_{\tree}(\xvec_i) \in \domain(\node)) \\
    &= \sum_{i=1}^\nobs \one(\sigma_{\tree}(\xvec_i)_j = a \land \sigma_{\tree}(\xvec_i) \in \domain(\node_{\lleft})) + \sum_{i=1}^\nobs \one(\sigma_{\tree}(\xvec_i)_j = a \land \sigma_{\tree}(\xvec_i) \in \domain(\node_{\rright}))\\
    &= c_{\node_{\lleft}}(j,a) + c_{\node_{\rright}}(j,a) \,,
\end{aligned} \label{eqn:counts-split}
\end{align}
where the second equality comes from the fact that $\domain(\node_{\lleft}) \cap \domain(\node_{\rright}) = \emptyset$ and $\domain(\node_{\lleft}) \cup \domain(\node_{\rright}) = \domain(\node)$.
Second, we show a following fact that will be useful in the proof:
\begin{align}
    \pi_{\anc(\node_{\lleft})}^\new \circ \tilde{\pi}_{\node}^{-1} 
    &= \pi_{\anc(\node_{[M-k]})}^\new \circ \tilde{\pi}_{\node_{[M-(k+1)]}}^{-1}  \label{eqn:sigma-tilde-reverse} \\
    &= (\tilde{\pi}_{N_{[0]}} \circ \tilde{\pi}_{N_{[1]}} \circ \cdots \circ\tilde{\pi}_{N_{[M-(k+2)]}} \circ \tilde{\pi}_{N_{[M-(k+1)]}} ) \circ \tilde{\pi}_{\node_{[M-(k+1)]}}^{-1}  \\
    &= \tilde{\pi}_{N_{[0]}} \circ \tilde{\pi}_{N_{[1]}} \circ \cdots \circ \tilde{\pi}_{N_{[M-(k+2)]}}  \\
    &= \pi_{\anc(\node_{[M-(k+1)]})}^\new = \pi_{\anc(\node)}^\new  \label{eqn:sigma-tilde-reverse-2}\,,
\end{align}
where the \eqref{eqn:sigma-tilde-reverse} and \eqref{eqn:sigma-tilde-reverse-2} are by \autoref{Lemma:local-reverse-order} 
Finally, we can prove the inductive step using the two facts above:
\begin{align}
    c_{\node^\new}
    &= c_{\node_{\lleft}^\new} + c_{\node_{\rright}^\new} \\
    &= \pi_{\anc(\node_{\lleft})}^\new[\tilde{c}_{\node_{\lleft}}] + \pi_{\anc(\node_{\rright})}^\new[\tilde{c}_{\node_{\rright}}] \\
    &= \pi_{\anc(\node_{\lleft})}^\new[\tilde{c}_{\node_{\lleft}}] + \pi_{\anc(\node_{\lleft})}^\new[\tilde{c}_{\node_{\rright}}] \\
    &= \pi_{\anc(\node_{\lleft})}^\new[\tilde{c}_{\node_{\lleft}} + \tilde{c}_{\node_{\rright}}] \\
    &= \pi_{\anc(\node_{\lleft})}^\new[\tilde{c}_{\node}^{\textnormal{init}}] \\
    &= \pi_{\anc(\node_{\lleft})}^\new[\tilde{\pi}_{\node}^{-1}[\tilde{c}_{\node}]] \\
    &= \pi_{\anc(\node)}^\new[\tilde{c}_{\node}]
\end{align}
where the first is by \eqref{eqn:counts-split}, the second is by the inductive hypothesis, the third is by noting that the left and right nodes have the same ancestors and defining $\pi_{\anc(\node_{\lleft})}^\new \equiv \pi_{\anc(\node_{\rright})}^\new$, the fourth is by the linearity of a permutation function, the fifth is by the definition $\tilde{c}^{\init}$ from Line 4 of \autoref{alg:learn-local-permutations}, the sixth is by the count permutation in Line 7 of \autoref{alg:learn-local-permutations}, and the last is by the composition of permutations and \eqref{eqn:sigma-tilde-reverse-2}.

\end{proof}

\subsection{Proof that algorithm produces a rank consistent TSP} 
\LearnLocalPermutations*
\begin{proof}
From \autoref{thm:relation-local-and-new-counts}, we know that $c_{\node^\new} = \pi_{\anc(\node)}^\new[\tilde{c}_{\node}]$.
We also note that the local counts are rank consistent by construction via Lines 6 and 7 of \autoref{alg:learn-local-permutations}, i.e., $\tilde{c}_\node \in \rcset$. 
Thus, we will show that applying $\pi_{\anc(\node)}^\new$ maintains the rank consistency property of $\tilde{c}_{\node}$ (similar to how a merge sort will retain the ordering when merging disjoint sets).

First, we can split the ancestor permutation into an independent permutation per feature:
\begin{align}
    \pi_{\anc(\node)}^\new[\tilde{c}_{\node}] = [\pi_{\anc(\node),0}^\new[\tilde{c}_{\node}(0)], \pi_{\anc(\node),1}^\new[\tilde{c}_{\node}(1)], \cdots]
\end{align}
Without loss of generality, we can focus on the $j$-th feature and expend the ancestor permutation based on \autoref{Lemma:local-reverse-order}:
\begin{align}
    \pi_{\anc(\node),j}^\new[\tilde{c}_{\node}(j)] 
    = \tilde{\pi}_{\node_{[0]},j} \circ \tilde{\pi}_{\node_{[1]},j} \circ \cdots \circ \tilde{\pi}_{\node_{[M-1]},j}[\tilde{c}_{\node}(j)] 
\end{align}
We will prove that the local permutations at every level for the feature $j$ will maintain rank consistency. Without loss of generality, let's denote this by $N \equiv N_{m}$, where $0 \leq m \leq M-1$ (this means all nodes are non-leaves).
We note that $\tilde{c}_\node^\init = \tilde{c}_{\node_{\lleft}} + \tilde{c}_{\node_{\rright}}$ from Line 4 of \autoref{alg:learn-local-permutations}.

We will first consider the case where the node split feature is not $j$, i.e., $j \neq s$.
If $j \neq s$, then $\domain_j(\node) = \domain_j(\node_{\lleft}) = \domain_j(\node_{\rright})$.
Thus, because the local counts are rank consistent on their domains, i.e., $\tilde{c}_{\node_{\lleft}}(j) \in \rcset(\domain_j(\node))$ and $\tilde{c}_{\node_{\rright}}(j) \in \rcset(\domain_j(\node))$, then $\tilde{c}^\init_{\node}(j) = \tilde{c}_{\node_{\lleft}}(j) + \tilde{c}_{\node_{\rright}}(j) \in \rcset(\domain_j(\node))$ because adding two vectors that are already ranked ordered will still be ranked ordered even when restricted to a subset of the vector based on the domain.
Now because $\tilde{c}^\init_{\node}(j)$ is already rank ordered, the learned permutation will be the identity, i.e., $\tilde{\pi}_{\node,j} = \pi_{\Id}$.
Clearly, the identity permutation will maintain rank consistency.

Now we consider the case where $j$ is the split feature, i.e., $j = s$.
If $j = s$, then the children domains are a partition of the parent's domain, i.e., $\domain(\node_{\lleft}) \cap \domain(\node_{\rright}) = \emptyset$ and $\domain(\node_{\lleft}) \cup \domain(\node_{\rright}) = \domain(\node)$.
Therefore, the initial local counts are either from the left or right but not both, i.e.,
\begin{align}
    \tilde{c}_{\node}^\init(j,a) = \begin{cases}
    \tilde{c}_{\node_{\lleft}}(j,a), & \textnormal{if }\, a \in \domain_j(\node_{\lleft}) \\
    \tilde{c}_{\node_{\rright}}(j,a), & \textnormal{if }\, a \in \domain_j(\node_{\rright}) \\
    \end{cases}
    \label{eqn:disjoint-counts}
\end{align}
We will now prove by contradiction that $\tilde{\pi}_\node$ maintains rank consistency.
Suppose $\tilde{\pi}_\node$ did not maintain rank consistency of its children, i.e., $\tilde{c}_{\node_{\lleft}}(j, \tilde{\pi}_\node^{-1}(a)) > \tilde{c}_{\node}(j, \tilde{\pi}_\node^{-1}(b))$ for $a < b$ and $a,b \in \domain(\node_{\lleft})$ where we can focus on the left node without loss of generality.
By the assumption on $\tilde{\pi}_\node$ and the fact that $\tilde{\pi}_\node$ sorts the counts from Line 6 of \autoref{alg:learn-local-permutations}, we know that $\tilde{c}^\init_{\node}(j, a) > \tilde{c}^\init_{\node}(j, b)$. 
And by combining this with the fact that $a,b \in \domain_j(\node_{\lleft})$ and the \eqref{eqn:disjoint-counts}, we can infer that $\tilde{c}_{\node_{\lleft}}(j, a) > \tilde{c}_{\node_{\lleft}}(j, b)$.
However, this contradicts the fact that we know $\tilde{c}_{\node_{\lleft}}(j, a) \leq \tilde{c}_{\node_{\lleft}}(j, b)$ because the local counts of the children are rank consistent bye construction, i.e., $\tilde{c}_{\node_{\lleft}}(j) \in \rcset(\domain(\node_{\lleft}))$.
Thus, an ancestor permutation on a split feature $j=s$ will also maintain rank consistency of its children.

Putting it all together, we know that $c_{\node}^\new = \pi_{\anc(\node)}^\new[\tilde{c}_{\node}]$ from \autoref{thm:relation-local-and-new-counts}, we know that $\tilde{c}_{\node} \in \rcset$ by Lines 6 and 7 of \autoref{alg:learn-local-permutations}, and we have just proven that $\pi_{\anc(\node)}^\new$ maintains the rank consistency of the input. Therefore, $\forall \node, \,\, c_{\node}^\new \in \rcset$.
\end{proof}

\subsection{Proof of \autoref{thm:algorithm-optimal}}
\AlgorithmOptimal*
\begin{proof}
The proof follows directly by combining the lemmas, i.e., that the algorithm will produce a TSP that has equivalent tree structure and that this TSP is rank consistent.
Then, we apply \autoref{thm:optimality-of-rank-consistent-tsps} to prove that this is indeed optimal given the tree structure.
\end{proof}

\section{More experimental details and results}
\label{exp_details}
\subsection{Modification of the Bipartite flow code}
\label{modify_bipartite}
Note that the BF model implemented in \cite{bricken_trentbrickpytorchdiscreteflows_2021}, had errors when training dataset with odd dimensions and dimensions greater than 2.
Thus, modifications were made to fix the bipartite model code (that are discussed in more details in the appendix). The model's initial embedding flow layers (commonly used for NLP sequence data) were replaced by a 5 linear layer network with batch normalization and ReLU activation functions on every intermediate output node. A skip connection was implemented between the input and the output for every coupling layer.
The BF model implemented in \cite{bricken_trentbrickpytorchdiscreteflows_2021}, had some bugs in the code (e.g., running the bipartite model for data dimensions $d>2$ would yield a runtime error).Thus, modifications were made to fix the bipartite model code. The model's initial embedding flow layers (commonly used for NLP sequence data) were replaced by a single hidden layer network with an activation function. A ReLU activation function was incorporated into the hidden layer to add non-linearity. The size of the hidden layer was proportional to the feature size times half the dimension $(\frac{1}{2}*k*d)$, since only half of the dimension would be mapped after the split. Each flow layer does a transformation on one of the splits, therefore, at least a paired flow layers (even number of flow layers) is required.

\subsection{More details about Gaussian Copula Synthetic Data Experiments}
\label{Copula}
The generating process can be summarized as follows:
We first generated data from a multivariate normal distribution, which can have strong dependencies between features, and normalized each feature by subtracting the mean and dividing by the standard deviation. 
Then, we applied the CDF of a standard normal distribution to all features independently---which creates uniform marginal distributions.
Finally, we applied the inverse CDF of a discrete distribution---which will generate discrete data; the discrete distribution could be a Bernoulli distributions with different bias parameters $p$ ($k=2$) or more generally a categorical distribution ($k\geq2$).
We then generated four datasets using the Gaussian copula model, with varying degrees of dependency among the dataset’s features, where the underlying Gaussian graphical model is a simple cycle graph. 
Specifically, we set the total correlation (a generalization of mutual information) of the Gaussian distribution (underlying the Gaussian copula), which can be solved in closed form for Gaussian distributions, to be $\{0,1,10,100\}$.
For all experiments, we used the inverse CDF of a Bernoulli distribution with parameters $p \in \{0.5, 0.3, 0.5, 0.2\}$ applied to each of the 4 dimensions separately. or a balanced dimensional transformation.

\subsection{Model space}
\label{model_space}
Below we list the models that we tried for 
$\textnormal{DTF}_{GLP}$, $\textnormal{DTF}_{RND}$ grouped under DTF, as well as the models we tried for AF and BF and DDF in an attempt to find the "best" model for all the experiments of section 4.

\begin{table}[!ht]
\centering
 \caption{Models tried for different exps. }
 \label{tab:real_data_combined}
\begin{tabular}{|c|c|c|c|c|c|}
\hline
\multicolumn{1}{|c|}{Exp} &
\multicolumn{1}{c|}{DTF} &
\multicolumn{1}{c|}{AF} &
\multicolumn{1}{c|}{BF}&
\multicolumn{1}{c|}{DDF}\\ 
\hline
 synthetic & \shortstack{ TSPs $\in \{1,2..,10\}$ \\ M $\in \{2,3,...,8\}$} & hidden layers $\in \{64, 128\}$ & \shortstack{$\alpha \in \{1\}$ \\ $\beta \in \{2, 16\}$} & \shortstack{hidden layers $\in \{128, 256, 512\}$ \\coupling layers $\in \{1,2,..5\}$  }\\ 
\hline
 Mushroom & \shortstack{ TSPs $\in \{1,2..,10\}$ \\ M $\in \{2,3,...,8\}$} & hidden layers $\in \{128\}$ & \shortstack{ $\alpha \in \{2, 4\}$ \\ $\beta \in \{2, 4, 8\}$} & \shortstack{hidden layers $\in \{128, 256, 512\}$ \\coupling layers $\in \{1,2,..5\}$}\\  
\hline
 MNIST & \shortstack{TSPs $\in \{1,2..,30\}$ \\ M $\in \{3,7,10\}$} & hidden layers $\in \{1568\}$ & \shortstack{ init $\in \{id, $w$\backslash$ o$ id\}$ \\ $\alpha \in \{2, 4, 8\}$ \\ $\beta \in \{1, 1/4, 1/16\}$}  & \shortstack{hidden layers $\in \{128, 256, 512\}$ \\coupling layers $\in \{1,2,..5\}$}\\ 
 \hline
 Genetic & \shortstack{TSPs $\in \{1,2..,20\}$ \\ M $\in \{3\}$} & hidden layers $\in \{1610\}$ & \shortstack{init $\in \{id, $w$\backslash$ o$ id\}$ \\ $\alpha \in \{2, 4, 8\}$ \\ $ \beta \in \{1, 1/4, 1/16\}$} & \shortstack{hidden layers $\in \{128, 256, 512\}$ \\coupling layers $\in \{1,2,..5\}$}\\
\hline
\end{tabular}
\end{table}

\paragraph{DDF's details: } We used the mlp architecture for all experiments of that we tested for the DDF model, and varied the number of coupling layers and hidden layers across the different experiments. We trained all models using 10 epochs for both the prior and the neural network training since increasing this number further didn't lead to a significant improvement in the results.

\paragraph{AF and BF details: }
For the AF and BF models, we use a linear layer architecture where 4 flow layers were used for the synthetic data experiments, and 6 flow layers were used for the mushroom, MNIST, and genetic data experiments experiments. We ran for a total of 200 epochs sampling 250 samples in each epoch. 

\paragraph{AF's model architecture }The AF model was constructed based on a Masked Autoencoder for Distribution Estimation (MADE) architecture, and a fixed number of hidden layers were decided for each experiment. By default, all small dimensional dataset ($d\leq10$) was set to 64 hidden layers and medium dimensional dataset ($10\leq d\leq100$) was set to 128 hidden layers. Very large dimensional dataset (MNIST \& Genetic) had hidden layers twice the size of its dimension. 

\paragraph{BF's model architecture: }The BF model had two different 5 linear layer architectural structures. The first BF model scales the initial input size $I$ by a constant $\beta$ and retains both the input and outputs size to $\beta * I$ for all the intermediate linear layers and reduces the size back to its input size at the last linear layer. The second model followed an autoencoder structure where the output of the first two linear layers reduces the input by a constant factor of $\alpha$. The third linear layer retains the size of the input and the last 2 linear layers increases the inputs by a constant factor of $\alpha$, thus rescaling the final layer output size back to its original input size. Moreover, for MNIST and Genetic datasets experiments the models had an extra preprocessing stage where the layers could be initial set to behave as an identity function seen in Table 4 (init $\in \{id, $w$\backslash$ o$ id\}$).

\paragraph{Other notes: } Our DTF model's independent base distribution $Q_{\zvec}$ is the marginal distribution calculated using the counts of each category and for each dimension of our output. The baseline model's prior distribution, however, is a randomized marginal distribution parameter that is optimized as the NLL is minimized in the deep learning model. For a better comparison with our DTF model -where Q is initialized to be the marginal on the input data- we include an initialization stage for the BF model that modifies the model parameter weights to make the first forward pass resemble an identity function. We also initialize the marginal distribution as the frequency counts of the input dataset. This helps the baseline model have the same NLL starting value as that of our DTF model. We include both the identity initialization and the random initialization for the BF baseline experiments for the MNIST dataset and the genetic dataset.

\subsection{The full table for synthetic datasets}
\label{app:syn_results_full}
Below we present the results of Table~\ref{tab:syn_exp_results} in more details.

\begin{table}[ht]
\caption{The full table version of Table~\ref{tab:syn_exp_results}}
\label{tab:syn_exp_results_full}
\begin{center}
\begin{small}
\begin{sc}
\begin{tabular}{p{0.18\linewidth}p{0.12\linewidth}p{0.12\linewidth}p{0.12\linewidth}p{0.12\linewidth}p{0.12\linewidth}}
\toprule
& AF & BF & DDF & 
$\textnormal{DTF}_{GLP}$ & $\textnormal{DTF}_{RND}$\\
\hline
{\small \textbf{8Gaussian}} & \multicolumn{5}{c}{} \\
$\textnormal{NLL}$ &  6.92 \tiny{($\pm$ 0.06)} & 7.21 \tiny{($\pm$ 0.09)} & 6.42 \tiny{($\pm$ 0.03)} & 6.5 \tiny{($\pm$ 0.03)} & 6.94 \tiny{($\pm$ 0.04)}\\
TT  & 155.9 \tiny{($\pm$ 2.2)}  & 231.6 \tiny{($\pm$ 5.2)} & 119.8 \tiny{($\pm$ 0.8)} & 7.3 \tiny{($\pm$ 0.1)} & 0.4 \tiny{($\pm$ 0.0)} \\
NumParams  &  253656 &  571116 & 114796 
& 5650 \tiny{($\pm$ 38)} & 23693 \tiny{($\pm$ 1428)} \\ 
\hline
{\small \textbf{COP-H}} & \multicolumn{5}{c}{} \\
$\textnormal{NLL}$ &  1.53 \tiny{($\pm$ 0.02)} & 1.47 \tiny{($\pm$ 0.06)} & 1.46 \tiny{($\pm$ 0.1)} & 1.33 \tiny{($\pm$ 0.02)} & 1.33 \tiny{($\pm$ 0.02)} \\
TT  & 10.7 \tiny{($\pm$ 0.2)}  & 13.2 \tiny{($\pm$ 0.2)} & 58.1 \tiny{($\pm$ 1.0)} & $\leq$0.1 \tiny{($\pm$ 0.0)} & 0.1 \tiny{($\pm$ 0.0)} \\
NumParams & 21024 & 14112 & 541720 & 34 \tiny{($\pm$ 1)} & 119 \tiny{($\pm$ 8)} \\

\hline
{\small \textbf{COP-M}} & \multicolumn{5}{c}{}  \\
$\textnormal{NLL}$ &  1.76 \tiny{($\pm$ 0.1)} &  1.62 \tiny{($\pm$ 0.05)}& 1.51 \tiny{($\pm$ 0.16)} & 1.4 \tiny{($\pm$ 0.02)} & 1.4 \tiny{($\pm$ 0.02)}\\
TT  & 10.6 \tiny{($\pm$ 0.02)}  & 13.3 \tiny{($\pm$ 0.06)} & 77.9 \tiny{($\pm$ 1.8)} & $\leq$0.1 \tiny{($\pm$ 0.0)} & 0.1 \tiny{($\pm$ 0.0)} \\
NumParams &  21024 & 14112 & 677148 & 43 \tiny{($\pm$ 1)} & 132 \tiny{($\pm$ 11)} \\

\hline

{\small \textbf{COP-W}} & \multicolumn{5}{c}{} \\
$\textnormal{NLL}$ &  2.42 \tiny{($\pm$ 0.02)} & 2.35 \tiny{($\pm$ 0.03)} & 2.29 \tiny{($\pm$ 0.07)} & 2.22 \tiny{($\pm$ 0.02)} & 2.22 \tiny{($\pm$ 0.02)} \\
TT  & 10.5 \tiny{($\pm$ 0.01)}  & 13.2 \tiny{($\pm$ 0.1)} & 77.3 \tiny{($\pm$ 1.7)}
& $\leq$0.1 \tiny{($\pm$ 0.0)} & 0.1 \tiny{($\pm$ 0.0)}  \\
NumParams & 21024 & 14112 & 677148 & 40 \tiny{($\pm$ 0)} & 162 \tiny{($\pm$ 5)} \\
\bottomrule
\end{tabular}
\end{sc}
\end{small}
\end{center}

\end{table}

\subsection{The model that was used for the results}
\label{models_used}

In Table~\ref{tab:model_used} we report the model that corresponded to the results reported in Tables~\ref{tab:syn_exp_results} and ~\ref{tab:real_exp_results}.

\begin{table*}[ht]
\caption{The model parameters for the results presented in tables~\ref{tab:syn_exp_results} and ~\ref{tab:real_exp_results}. nH refers to the number of hidden layers, nC refers to the number of coupling layers, nTSP referes to the number of TSPs and M is the maximum depth }
\label{tab:model_used}
\centering
\begin{tabular}{p{0.12\linewidth}p{0.12\linewidth}p{0.12\linewidth}p{0.12\linewidth}p{0.12\linewidth}p{0.12\linewidth}}
\hline
& AF & BF & DDF  & $\textnormal{DTF}_{GLP}$ & $\textnormal{DTF}_{RND}$\\
\hline

$\textnormal{8 Gaussian}$ & nH = 128 & \shortstack{$\alpha = 1$\\ $\beta = 2$} & \shortstack{nC = 2\\ nH = 128} & 
\shortstack{nTSP = 9\\ M = 2} & \shortstack{nTSP = 10\\ M = 8}\\
\hline

$\textnormal{COP-H}$ & nH = 64 & \shortstack{$\alpha = 1$\\ $\beta = 16$} & \shortstack{nC = 4\\ nH = 256} & 
\shortstack{nTSP = 2\\ M = 2} & \shortstack{nTSP = 10\\ M = 2}\\
\hline

$\textnormal{COP-M}$ & nH = 64 & \shortstack{$\alpha = 1$\\ $\beta = 16$} & \shortstack{nC = 5\\ nH = 256} & 
\shortstack{nTSP = 2\\ M = 2} & \shortstack{nTSP = 10\\ M = 2}\\
\hline

$\textnormal{COP-W}$ & nH = 64 & \shortstack{$\alpha = 1$\\ $\beta = 16$} & \shortstack{nC = 5\\ nH = 256} & 
\shortstack{nTSP = 2\\ M = 2} & \shortstack{nTSP = 10\\ M = 2}\\
\hline

$\textnormal{Mushroom}$ & nH = 128 & \shortstack{$\alpha = 4$\\ $\beta = 8$} & \shortstack{nC = 5\\ nH = 512} & 
\shortstack{nTSP = 8\\ M = 6} & \shortstack{nTSP = 10\\ M = 7}\\
\hline

$\textnormal{MNIST}$ & nH = 1568 & \shortstack{$\alpha = 4$\\ $\beta = 1/16$} & \shortstack{nC = 5\\ nH = 256} & 
\shortstack{nTSP = 21\\ M = 10} & \shortstack{nTSP = 30\\ M = 3}\\
\hline

$\textnormal{Genetic}$ & nH = 1610 & \shortstack{$\alpha = 4$\\ $\beta = 1/16$} & \shortstack{nC = 5\\ nH = 128} & 
\shortstack{nTSP = 6\\ M = 3} & \shortstack{nTSP = 20\\ M = 3}\\
\hline

\end{tabular}
\end{table*}

\subsection{GPU training times comparison}
\label{GPU_train_times}
In section 4 of the main paper, we presented all training times on CPU for our model vs AF or BF. In Table~\ref{tab:GPU_train_times}, we present our model's training time on CPU vs the AF and BF training times on GPU.

\begin{table}[!ht]
\centering
\caption{Training times for baselines (GPU) and our approach (DTF). Our random DTF ($\textnormal{DTF}_{RND}$) trained on a CPU is faster than other methods \emph{even when baselines are trained on a GPU}.
}
\label{table:comparison}

\begin{tabular}{c | c | c | c  |c  |c} 
\hline\hline 
Model & 8Gaussian & COPH & Mushroom & Genetic & MNIST \\ [0.5ex] 
\hline 
AF (GPU) & 42.2 \tiny{$\pm$ 3.5} & 14.6 \tiny{ $\pm$ 0.4} & 7.7 \tiny{ $\pm$ 1.0}  & 23.8 \tiny{ $\pm$ 0.9}  & 305.6 \tiny{ $\pm$ 31.4}  \\
BF (GPU) & 135.8 \tiny{ $\pm$ 0.2} & 20.9 \tiny{$\pm$0.3} & 6.0 \tiny{ $\pm$ 1.3}  & 38.0 \tiny{ $\pm$ 5.4} & 308.7 \tiny{ $\pm$ 4.1}  \\
DDF (GPU) & 79.7 \tiny{$\pm$ 0.8} & 48.8 \tiny{ $\pm$0.9} & 75.2 \tiny{ $\pm$ 1.0}  & 29.5 \tiny{$\pm$1.1}  & 334.3 \tiny{ $\pm$ 8.7}  \\
$\textnormal{DTF}_{GLP}$ & 7.3 \tiny{ $\pm$ 0.1} & \textbf{$\leq$ 0.1 \tiny{ $\pm$ 0.0}} & 9.9\tiny{$\pm$ 0.2}  & 411.5 \tiny{ $\pm$ 2.3}  & 5213.7 \tiny{ $\pm$ 204.9} \\
$\textnormal{DTF}_{RND}$ & \textbf{0.4 \tiny{$\pm$ 0.0}} & \textbf{$\leq$ 0.1 \tiny{$\pm$ 0.0}} & \textbf{0.5 \tiny{$\pm$ 0.0} } & \textbf{5.9 \tiny{$\pm$ 0.0}}  & \textbf{105.6 \tiny{$\pm$ 0.1}} \\
\hline
\end{tabular}
\label{tab:GPU_train_times}
\end{table}

\subsection{Exploration of the samples obtained from sampling the learned distribution}
\label{visuals}
\paragraph{Discretized Gaussian Mixture Model}
To visualize samples that can be generated from the distribution using our model, 
we trained our best DTF models on the 8Gaussian dataset. We then generated 1024 samples that were plotted in  
Figure~\ref{fig:DTF_GLP_gauss_mix} for $\textnormal{DTF}_{GLP}$ and in Figure~\ref{fig:DTF_gauss_mix_rnd} for $\textnormal{DTF}_{RND}$. The original data distribution is plotted in Figure~\ref{fig:orig_gauss_mix}.
We also sampled 1024 samples generated from the distributions that was learned from the best AF, BF, and DDF models and the results of these are plotted in Figures~\ref{fig:AF_toy}, ~\ref{fig:BF_toy} and ~\ref{fig:DDF_samples} respectively.

We notice that the samples that were generated from the distribution we learned from the $\textnormal{DTF}_{GLP}$ model mirror the original samples of data better than the samples generated from the learned distribution of the AF or BF models and is of similar quality compared to those obtained from DDF.

\begin{figure}[ht]
     \centering
     \begin{subfigure}[h]{0.3\textwidth}
         \centering
         \includegraphics[width=\textwidth]{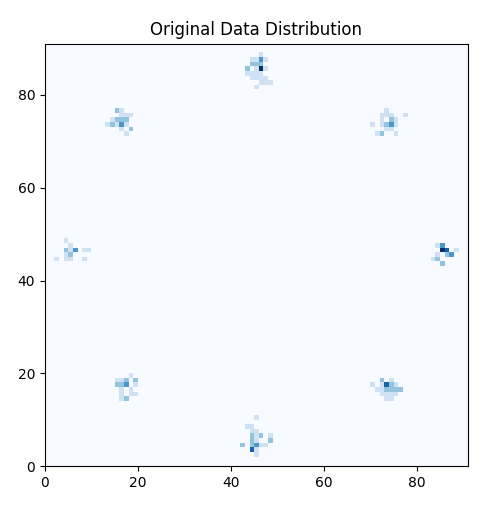}
         \caption{Original data distribution.}
         \label{fig:orig_gauss_mix}
     \end{subfigure}
     \hfill
     \begin{subfigure}[h]{0.31\textwidth}
         \centering
         \includegraphics[width=\textwidth]{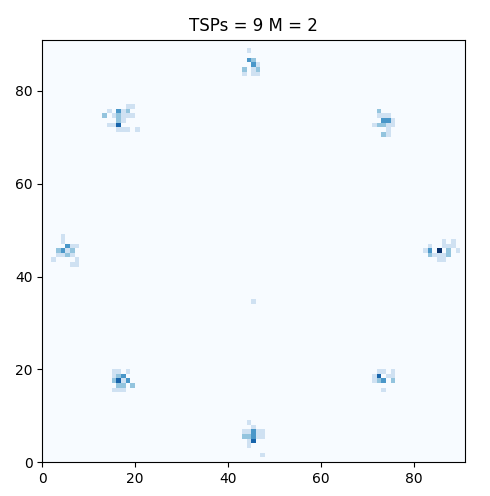}
         \caption{\shortstack{Samples from learned distribution\\ for $\textnormal{DTF}_{GLP}$.}}
         \label{fig:DTF_GLP_gauss_mix}
     \end{subfigure}
    \hfill
     \begin{subfigure}[h]{0.32\textwidth}
         \centering
         \includegraphics[width=\textwidth]{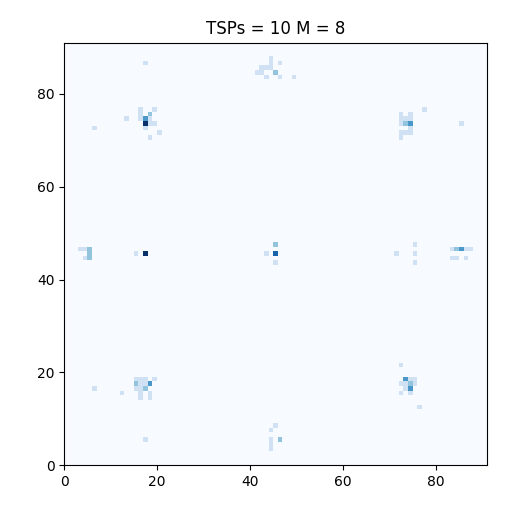}
         \caption{\shortstack{Samples from learned distribution\\ for $\textnormal{DTF}_{RND}$.}}
         \label{fig:DTF_gauss_mix_rnd}
     \end{subfigure}
     \begin{subfigure}[h]{0.29\textwidth}
         \centering
         \includegraphics[width=\textwidth]{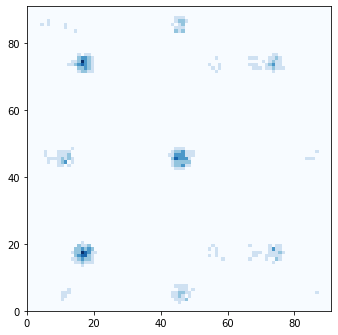}
         \caption{\shortstack{Samples from learned distribution\\ for AF.}}
         \label{fig:AF_toy}
     \end{subfigure}
    \hfill
    \begin{subfigure}[h]{0.29\textwidth}
         \centering
         \includegraphics[width=\textwidth]{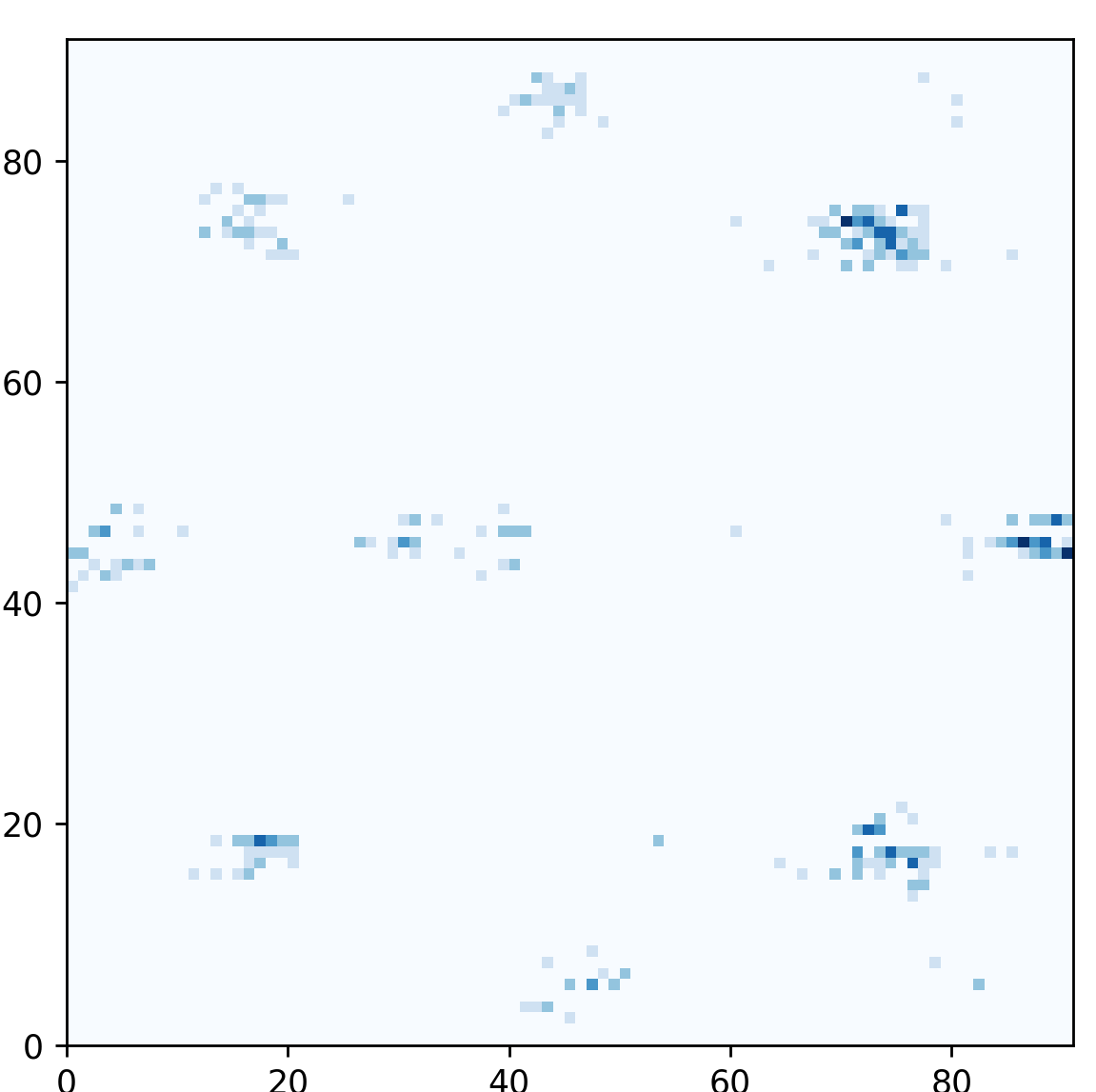}
         \caption{\shortstack{Samples from learned distribution\\ for BF.}}
         \label{fig:BF_toy}
     \end{subfigure}
     \hfill
     \begin{subfigure}[h]{0.31\textwidth}
         \centering
         \includegraphics[width=\textwidth]{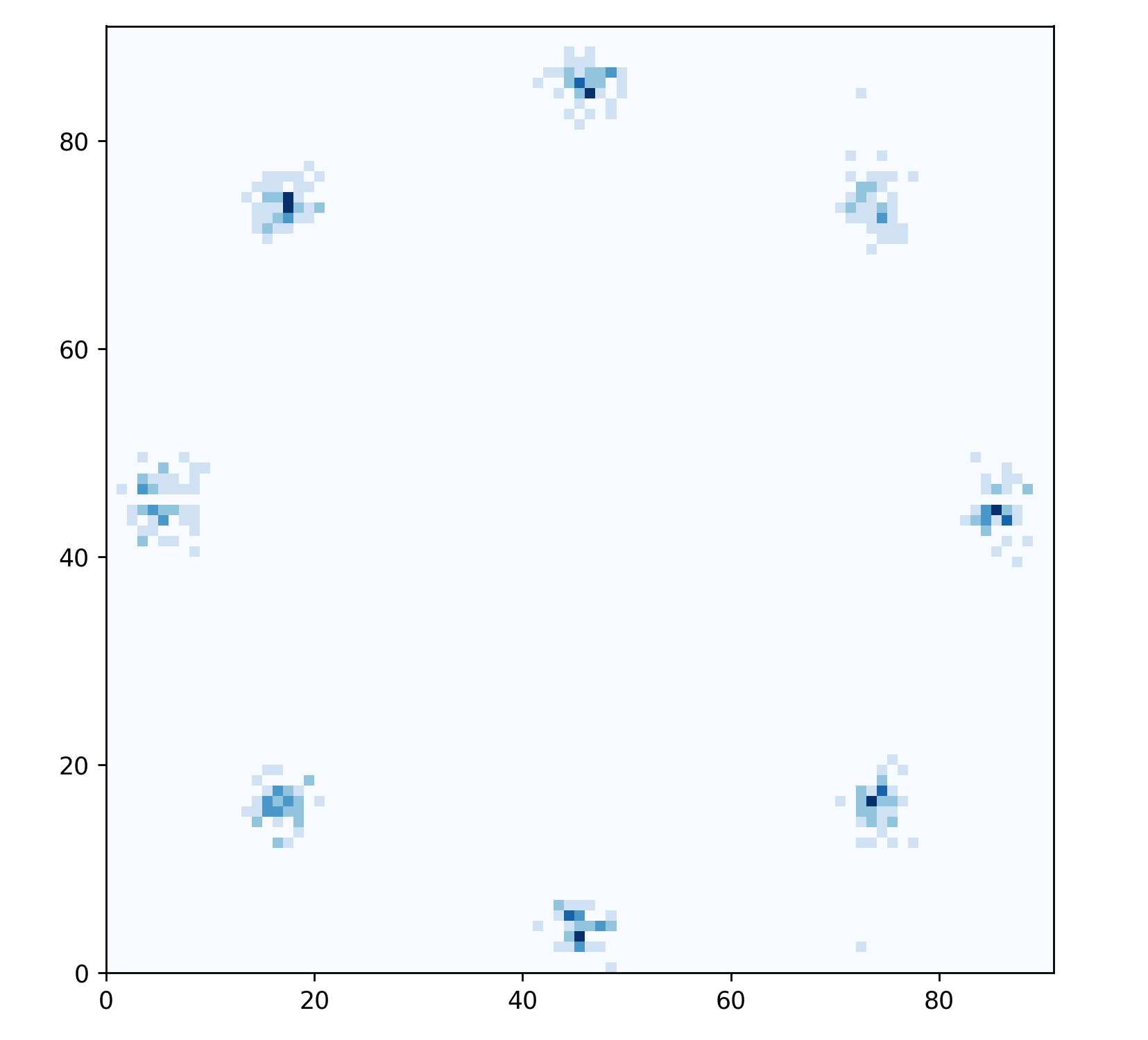}
         \caption{\shortstack{Samples from learned distribution\\ for DDF.}}
         \label{fig:DDF_samples}
     \end{subfigure}
\end{figure}

\end{document}